\documentclass[11pt]{amsart}
\usepackage[margin=1.0in]{geometry}
\usepackage[lite]{amsrefs}
\usepackage{amssymb}
 \usepackage{graphicx}
 \usepackage[colorlinks,linkcolor=red, citecolor=blue]{hyperref}
\usepackage[capitalize]{cleveref}
\numberwithin{equation}{section}
\usepackage{tcolorbox}
\usepackage[normalem]{ulem}
\crefformat{equation}{(#2#1#3)}

\usepackage{mydefs_stoc}

\title{A PDE-Based Analysis of the Symmetric Two-Armed Bernoulli Bandit}

\author{Vladimir A. Kobzar}\address{Department of Applied Physics and Applied Mathematics, Columbia University, 
New York, NY} 
  \email{vak2116@columbia.edu} 
\author {Robert V. Kohn}\address{Courant Institute of Mathematical Sciences, New York University, New York, NY }
  \email{kohn@cims.nyu.edu}

\begin{document}

\begin{abstract}%
This work addresses a version of the two-armed Bernoulli bandit problem where the sum of the means of the arms is one (the symmetric two-armed Bernoulli bandit).  In a regime where the gap between these means goes to zero as the number of prediction periods approaches infinity,  i.e., the difficulty of detecting the gap increases as the sample size increases,  we obtain the leading order terms of the  minmax optimal regret and pseudoregret for this problem by associating each of them with a solution of a linear heat equation. Our results improve upon the previously known results; specifically, we explicitly compute these leading order terms in three different scaling regimes for the gap. Additionally, we obtain new non-asymptotic bounds for any given time horizon.   Although  optimal player strategies are not known for more general bandit problems, there is significant interest in considering how regret accumulates under specific player strategies, even when they are not known to be optimal. We expect that the methods of this paper should be useful in settings of that type. 
\end{abstract}

\maketitle



\section{Introduction}
\label{sec:intro}

The \emph{multi-armed bandit} is a classic sequential prediction problem. At each round, the predictor (\emph{player}) selects a probability distribution from a finite collection of  distributions (\emph{arms}) with the goal of minimizing the difference (\emph {regret}) between the player's rewards sampled from the selected arms and the rewards of the best performing arm at the final round.  The player's choice of the arm and the reward sampled from that arm in that round are revealed to the player, and this prediction process is repeated until the final round. 

Since the rewards of the arms that are not sampled are not revealed to the player, this is an \emph{incomplete information} problem.  This leads to a principal challenge in devising player strategies for multi-armed bandits: balancing exploration of different arms with the exploitation of the information gathered during the earlier periods.  However, in the case of a two-armed Bernoulli bandit where the arms are distributed symmetrically, i.e., each arm is distributed independently according to a Bernoulli distribution and the sum of the means of the arms is one (\emph{symmetric two-armed Bernoulli bandit}), this challenge is not present. In this case, sampling from one arm is statistically equivalent to sampling from the other arm. 

The optimal player strategy in this setting is, perhaps, not difficult to guess; but we appear to be the first to give a proof of its optimality in the \emph{minimax} setting. Also, even in this simplified setting, the incomplete information aspect of the problem is remains, and the optimal regret has not been determined previously.  Accordingly, we develop a fresh PDE-based perspective on the symmetric two-armed Bernoulli bandit problem and apply it to determine the leading order term of optimal regret  when the gap between these means of the arms goes to zero as the number of prediction periods approaches infinity,  i.e., the difficulty of detecting the gap increases as the sample size increases.

Although optimal player strategies are not known for most other bandit problems, there is significant interest in considering how regret accumulates under  specific player strategies, even when they are not known to be optimal.  We expect that the methods of this paper should be useful in settings of  that type. Accordingly our primary algorithmic contribution is  a methodological advance, which augments the toolkit one can bring to bear on any bandit problem once the (potentially suboptimal) player's strategy has been fixed.
 
Stochastic bandits can be viewed as an interaction between an ``adversary'' that sets the distributions of the arms at the start of the game and the player who plays according to a specific strategy. In the simplified setting of the symmetric two-armed bandit, our methods allow us to obtain a rather complete understanding of how the regret achieved by the optimal player strategy depends on (a) the number of time steps, and (b) the gap between the means of the two arms.  Although the power of our ``adversary'' is  restricted to setting the gap between the arms, there appears to be essentially no method in the literature that allows one to evaluate the regret corresponding to various gap regimes except for the fixed gap and the gap that scales as $\Theta(1/\sqrt {T})$ where $T$ is the number of prediction periods. Our methods  allow for the first time to determine the leading order behavior of the regret in all other scaling regimes for the gap.

While the case of  general bandits is more challenging, since the player needs to balance exploration and exploitation, there are more realistic  settings than the symmetric two-armed bandit in which exploration is \emph{not} needed.\footnote{One may ask if bandit-type problems that do not require exploration should be categorized as ``bandits". The incomplete information aspect of the problems described in the paragraph accompanying this footnote led to highly nontrivial algorithmic questions  despite the lack of exploration. Accordingly, consistently with those references we shall also refer to the present simplified problem as a ``bandit" problem.} For example, reference \cite{feldman} considered a Bayesian two-armed bandit where each arm is distributed according to an arbitrary probability distribution; the special feature of that problem is that both distributions are known to the player, although the player does not know which distribution is associated with each arm. This reference showed that the optimal player in that setting is myopic.  Reference \cite{rodman} further showed that the myopic player is optimal in the Bayesian $k$-armed bandit setting where the player knows that one arm has distribution $P$ (but does not know which arm) and all the other arms have the same distribution $Q$ (different from $P$).\footnote {See also reference \cite{zaborskis} that showed the same result  restricted to Bernoulli distributions.} One important application of the problem described in the previous sentence is that it leads to lower bounds for the general $k$-armed bandit, where the player has no special information about the arms.\footnote{See, e.g., Theorem 3.5 in reference \cite {bubeck_book}.} 

The  minimax optimal regret and pseudoregret  we determine in the symmetric two-armed bandit setting  lead to new regret and pseudoregret lower bounds in the general two-armed bandit setting. Existing nonasymptotic  lower bounds rely on information theory, in particular Pinsker's inequality, to bound below the  (pseudo)regret in certain symmetric bandit problems, which lead to lower bounds in the general bandit problems.  (We further discuss these lower bounds later in this section.)  Our results lead to new nonasymptotic lower bounds established without appealing to information theory in the two-armed setting.  We  hope that our methods will make progress towards better lower bounds in general $k$-armed bandit problems. 

Let $a(j)=(a_1, a_2)$ refer to a pair of distributions (\emph{arms}) where arm $j$ (the \emph{safe} arm) is assigned $0$ with probability $\frac{1-\epsilon}{2}$ and $1$ with probability $\frac{1+\epsilon}{2}$ independently from the  other arm and the history, and the other arm $i$ (the \emph{risky} arm) is assigned $0$ with probability $\frac{1+\epsilon}{2}$ and $1$ with probability $ \frac{1-\epsilon}{2}$ also independently. This work studies the following problem.
\begin{tcolorbox}
 \emph{The symmetric two-armed bandit:} In each period $t$ starting from $-T$ until $-1$:
 \begin {enumerate}
 \item  The player determines how to sample the arms by selecting a discrete probability distribution $p_t$ over the two arms.
\item The rewards $g_t: = (g_{1,t}, g_{2,t})$ are sampled  from $a(j)$, as defined above, and the player's choice of the arm $I_t \in [2]$ is sampled  from $p_t$ independently of $g_t$. 
 \item  This choice $I_t$ and the reward of the chosen arm $g_{I_t,t} $  are revealed to the player. 
 \end{enumerate} 
 \end{tcolorbox}
\emph{We denote the time $t$ by nonpositive integers such that the starting time is $-T \leq -1$ and the final time is zero.} This convention is convenient because it will lead to the relevant value functions of the game being dependent on $t$ instead of $T-t$ had we set the starting time to 0 and the final time to $T$.

Although the identities of the safe and risky arms are never revealed to the player,  the player knows that the distribution of the arms is symmetric.\footnote{As the analysis below shows, an optimal player is the same for all feasible values of the gap $\epsilon$. Therefore, the player would not get any additional advantage if the numerical value of the gap were revealed to her.} We also denote the \emph {accumulated and instantaneous regret} by 
\begin{align*}
x_t := \sum_{ \tau<t} r_\tau ~\text{and} ~r_{\tau}  := g_{\tau}-g_{I_{\tau}, \tau} \mathbbm 1,
\end{align*}
 respectively. (These include rewards that have not been revealed to the player.) The associated final-time expected regret, or simply the \emph{regret},  is given by the iterated expectation
\begin{align}
&R_T(p,a(j)) :=  \mathbb E_{\subalign {&I_{-T} \sim p_{-T}\\ &g_{-T} \sim a(j)}} \Big [ \mathbb E_{\subalign {&I_{-T+1} \sim p_{-T+1}\\ &g_{-T+1} \sim a(j)}} \Big [ \dotso  \Big [ \mathbb E_{\subalign {&I_{-1} \sim p_{-1}\\ &g_{-1} \sim a(j)}}  \max_{ i \in [2] }  \sum_{ t = -T}^{-1}  (g_{i,t}-g_{I_{t}, t} \mathbbm 1) \Big ]  \dotso \Big ]  \Big ], \label{eq:regret}
\end{align}
which we denote succinctly as
\[
 \mathbb E_{ p,a(j)}   \max_{ i \in [2] } x_{i,0}.
 \]
The player strategy  $p = (p_{-T}, ..., p_{-1})$ is specified for every prediction period where each $p_t = (p_{1,t}, p_{2,t})$ is a discrete probability distribution over two arms.  This distribution can in principle be a function of all  information available to the player at time $t>-T$ (the history), i.e., $p_t \equiv p_t(H_{t-1})$ where
\begin{align}
H_{t-1}:=(I_{-T:t-1}, g_{I_{-T:t-1}}),\label{eq:history}
\end{align}
denotes the history,  $I_{-T:t-1} := I_{-T},..., I_{t-1}$ denotes the prior samples of the arms and $g_{I_{-T:t-1}} := g_{I_{-T},-T},..., g_{I_{t-1},t-1}$ denotes the previously revealed rewards.

 Note that the accumulated regret and instantaneous regret are vectors while the final-time expected regret is a scalar.  The player's objective is to minimize the  final-time expected regret for the choice of the safe and risky arms that maximizes this regret. Accordingly, a \emph {minimax optimal player} $p^*$ is a minimizer of the \emph{minimax regret}
\begin{align}
 \label{eq:minimax_regret}
 R^*_T:= \min_p \max_{j \in[2]}  R_T(p,a(j)) 
 \end{align}
 where the set of feasible $p$ is given in the previous paragraph. (We will refer to this player $p^*$ as simply an \emph{optimal player} when the context is clear.)
 
The \emph{suboptimality parameter or gap} of the arms is given by $\eps = \mu_j - \mu_i$ where $\mu_j$ and $\mu_i$ are the means of the safe and the risky arms, respectively.  We consider several scaling regimes where $\epsilon$ approaches zero as the number of prediction periods $T$ goes to infinity.
   
 Reference \cite{bather} considered the Bayesian version of our problem in the context of the following hypothesis test.  Let  the prior distribution be defined by assigning equal probabilities to 
\[
H_1: \mu_1 = \frac{1}{2} (1+ \eps), ~ \mu_2 = \frac{1}{2} (1- \eps), ~\text{and}~ H_2:  \mu_1 = \frac{1}{2} (1- \eps), ~
\mu_2 = \frac{1}{2} (1+ \eps)
\]
The expected number of times the inferior treatment (the risky arm $i$) is chosen is given by \textit{pseudoregret} $\bar R_T$ (also denoted as \textit{weak regret})
 \begin{align}
 \label{eq:pseudoregret1}
\bar R_T( p,a)= \eps \mathbbm E_{p,a(j)}  s_i 
 \end{align}
 where the expectation is computed similarly to \cref{eq:regret} and  $s_i$  denotes  the number of times the risky arm $i$ was sampled by the player.  Accordingly, a sampling rule that minimizes the expected number of times the inferior treatment is chosen leads to the \emph{Bayesian symmetric two-armed Bernoulli  bandit} problem: it has the same definition as the symmetric two-armed Bernoulli bandit  above, except that the index of the safe arm $j$ is sampled from a prior distribution over $\{1,2\}$  and the (Bayes) optimal player is a minimizer of the  \emph{Bayesian pseudoregret} (also called Bayes risk). In the case of the uniform prior, the Bayesian pseudoregret is given by
 \begin{align}
 \label{eq:baysian_pseudoregret}
\bar R_T^B = \min_p  \mathbb E_{j \sim \text{Unif}( \{1,2\})}  \bar R_T(p,a(j)) 
 \end{align}
where the set of feasible $p$ is the same as in the setting of the minimax regret above. 

For either choice of the safe arm, the distribution $a_{1}$ of arm 1 is the same as $1-a_{2}$, where $a_{2}$  is the distribution of the second arm. Thus, the player will get the same information about the means of both distributions by sampling either arm. Accordingly a success observed in any trial with arm 1 is equivalent to a failure observed from arm 2,  and the information  derived from any sequence of trials does not depend on the sampling rule.  

Let the revealed \emph{cumulative rewards} of arm $i$ be given by 
\[
G_i = \sum_{\tau <t}  g_{i,\tau}  \mathbbm 1_ {I_\tau =i},
\] 
 Reference \cite{bather} determined that the following player that selects the arm with the highest posterior probability of being the safe one given the revealed rewards (\emph{myopic player}) is Bayes optimal under the uniform prior. 
 \begin{tcolorbox} \emph{Myopic player $p^m$} for the two-armed Bernoulli bandit problem is
\begin{align*}
p^m = 
\begin{cases}
 (1,0) ~\text {if}~2G_1 -2G_2+s_2-s_1>0   \\
    \left(\frac{1}{2},\frac{1}{2}\right) ~\text {if}~2G_1 -2G_2+s_2-s_1=0\\
        (0,1) ~\text {if}~2G_1 -2G_2+s_2-s_1<0
        \end {cases} 
\end{align*}   
where $G_i$ and $s_i$ are defined above.
 \end{tcolorbox}
 Reference \cite{bather} also determined the leading order term of the above-mentioned Bayesian pseudoregret \eqref{eq:baysian_pseudoregret} to be $.265 \sqrt{T}$ (which corresponds to $.530 \sqrt{T}$ in the centered version of the problem we consider below).  Since an expectation is less or equal to the maximum,  $\bar R_T^B$ bounds below the \emph{minimax pseudoregret} given by
 \begin{align}
 \label{eq:minimax_pseudoregret}
\bar R_T^*  =\min_p \max_{j \in [2]} \bar R_T(p,a(j)).
 \end{align}
  Also, since \eqref{eq:pseudoregret1} can be equivalently expressed as
 \begin{align}
 \label{eq:pseudoregret}
\bar R_T(p,a) =   \max_{ i \in [2] } \mathbb E_{ p,a}  x_{i, 0},
 \end{align}
we have $\bar R_T(p,a) \leq  R_T(p,a)$ for any $p$ and $a$ as a result of exchanging the maximum with the expectation. Therefore the Bayesian pseudoregret  $\bar R_T^B$  also bounds the minimax regret $R_T^*$ below.  The Bayesian pseudoregret determined in  \cite{bather}  corresponds to the regime in which the gap between the means of the arms $\eps$ is a constant multiple of $T^{-\frac{1}{2}}$ (\emph{medium gap}) where $T$ is the number of prediction periods.

Although it is well-known that one can achieve $O(\sqrt{T})$-regret and pseudoregret in this (and more general) bandit settings, the exact constant inside the $O(\cdot)$ was not previously known in the minimax setting; also regret and pseudoregret  have not been previously determined across different scaling regimes of the gap. We obtain such results as well as eliminate several other conceptual barriers towards a more complete understanding of the regret under various scaling regimes of the gap by applying   PDE-based methods to the symmetric two-armed bandit model. Our principal conceptual advances are the following:
\begin {enumerate} 
\item We show that the  optimal player  in the symmetric two-armed bandit problem in the minimax setting is the same as in the Bayesian setting described above. We appear to be the first to give a proof of its optimality in the minimax setting, although its optimality in the Bayesian setting is known. This allows us to apply methods based on partial differential equations (PDE) to compute the regret and pseudoregret in the minimax setting. Thus, our methods make progress towards unifying the analysis of Bayesian and  minimax regret on the one hand, and unifying the analysis of regret and pseudoregret, on the other hand. 
\item   Since the optimal player is discontinuous as a function of revealed gains, the spatial derivatives of the solutions of the relevant PDEs are also discontinuous. While this  discontinuity does not affect the leading order term of the regret, it affects the discretization error. We are able to optimize this discontinuity to minimize this error. 
\item We determine the  minimax optimal regret and pseudoregret in the symmetric bandit setting, which  leads to new regret and pseudoregret lower bounds in the general two-armed bandit setting. While existing nonasymptotic  lower bounds rely on information theory,  as further discussed below, our results lead to new lower bounds established by more elementary techniques.  
\end{enumerate}
 These advances not only provide a fresh perspective on the symmetric two-armed bandit problem, but also allow us to improve on the existing bounds.
\begin{enumerate}
 \item  We show that the previously known leading order term of pseudoregret obtained in the Bayesian setting in \cite{bather}, corresponding to the medium gap regime, matches that in the minimax setting by associating the minimax pseudoregret with an explicit solution of a linear heat equation (\cref{sec:regret}).  In the hypothesis testing framework described above, our results extend to the minimax  setting the guarantee on the expected number of times the inferior treatment (risky arm) is chosen.  
\item  Although the optimal player is the same in the regret and pseudoregret settings, in the medium gap regime, the exact value of $\eps$ that inflicts the optimal regret is smaller than the one that inflicts the optimal pseudoregret, albeit still strictly larger than zero, which we believe has not been demonstrated previously.  Specifically, the largest regret of $.286\sqrt{T}$ (or $ .572\sqrt{T}$ in the equivalent centered problem described below) is achieved when the safe arm has mean $1/2 + .353/\sqrt {T}$ (or $.707/\sqrt{T}$ in the centered problem) (\cref{fig:c_gamma}).\footnote{These prefactors are rounded to 3 decimal places.}     In the hypothesis testing framework of \cite{bather}, the regret represents the expected difference between the outcomes of the better fixed treatment in hindsight and the outcomes of the sequence of treatments chosen by the player. 
\item Our methods also obtain the leading terms of the regret and pseudoregret if  the gap approaches zero (a) faster than a constant multiple of $T^{-\frac{1}{2}}$ (\emph{small gap}) or (b) slower than a constant multiple of $T^{-\frac{1}{2}}$ 
(\emph{large gap}) (\cref{tab:summary}).
\item In the small gap regime, the regret does not depend on the gap and in particular, it is the same as in the regime where the gap is zero. On the other hand, the optimal pseudoregret is $\eps T/2$ (or $\eps T$ in the centered version of the problem), which would be the same if   the player naively sampled each arm an equal number of times. This establishes (again without appealing to information-theoretic tools) that the optimal player cannot detect the gap in this regime.
\item Our methods also provide new non-asymptotic guarantees in each of the three gap regimes (\cref{sec:regret}, \cref{sec:pseudo_regret} and \cref{tab:summary}).
\end {enumerate}

 PDE-based methods have been previously applied to other bandit problems. For example, references \cite{Chernoff68, Chang87, Lai88} used free-boundary problems involving the heat equation to study bandit problems in the \emph{fixed} gap regime.  These bounds typically scale as $O(\frac{1}{\eps} \log T)$ and therefore do not guarantee $O(\sqrt {T})$ regret whenever the gap $\epsilon$ approaches zero faster than a constant multiple of $T^{-\frac{1}{2}}  \log T$.\footnote{See also reference \cite{Lai05} for a survey of these and related results.}   Reference \cite{KW23} considered  the diffusion limit of the Thompson sampling strategy in the general bandit setting, and among other results, upper bounded the pseudoreget associated with this strategy in the  two-armed bandit setting in the large gap regime.\footnote {Since Thompson sampling is not necessarily an  optimal strategy in the present setting, in  \cref{sec: existing bounds} we confirm that the minimax regret we obtain for the symmetric two armed bandit in the large gap regime satisfies the upper bound in  \cite{KW23}, and therefore our results are consistent with that reference.} To our knowledge, the present paper is the first application of a PDE-based methods to guarantee $O(\sqrt {T})$ minimax regret and pseudoregret  in a bandit problem when the gap approaches zero at an \emph{arbitrary} rate, i.e., the difficulty of detecting the gap increases \emph{arbitrarily} as the sample size increases.

 Our methods involve identifying a PDE whose solutions approximate the final time regret (asymptotically, in certain regimes as the number of time steps tends to infinity and the parameter $\eps$ tends to zero). It is easy to explain, at a conceptual level, why a PDE-based method is useful. Indeed, our symmetric two-armed bandit problem has the feature that the optimal player strategy is known, and it depends on the history in a very simple way. Therefore (as we shall explain), the evolution in time of the (optimal) player's regret can be viewed as a random walk in a suitable state space. Since we are interested in the properties of this random walk over long times, one approach would be to consider a suitable scaling limit (in the same way that a simple random walk on a lattice can be studied by considering Brownian motion). For example, a Hamilton-Jacobi-Bellman PDE emerged in reference \cite{zhu22} from applying a scaling argument in the context of considering optimal player strategies for $k$-armed Bayesian bandits.\footnote{In that general setting, the optimal player is not  known explicitly, and while the  PDE-based model is supported by extensive numerical experiments, convergence of the value function of the discrete bandit problem to the PDE solution, as well as explicit regret bounds in different scaling regimes, have not yet been obtained analytically.  Our PDE-based methods are aimed to make progress towards achieving those results.}
 
In the present setting a more elementary alternative to the scaling argument is also available, namely: the backward Kolmogorov equation of the scaling limit is easy to guess; since  the expected value of the random walk is like a discrete-time numerical scheme for this PDE, the fact that the PDE solution and this value function are close can be shown using little more than Taylor expansion. Our analysis uses this more elementary approach. Its execution is complicated by the fact that the solution of our PDE is not smooth -- rather, it is piecewise smooth and at most $C^1$ in the spatial coordinates, with bounded second-order derivatives. But the execution is simplified by the fact that the solution can be found explicitly; therefore the error terms introduced by Taylor expansion have  explicit estimates.  
 
 The symmetric two-armed Bernoulli bandit we examine is a restriction to $k=2$ of the $k$-armed bandit distribution that provides essentially the only known lower bound for the general  \textit{$k$-armed  stochastic bandit} problem. In that setting there are  $k$ probability distributions (arms) $a=(a_1, \dotso, a_k)$, and the safe arm is chosen uniformly at random at the start of the prediction process. In each period $t$ , the player determines which of the $k$ arms to follow by selecting a discrete probability distribution $p_t \in \Delta_k$;  the arms' rewards $g_t$ and the player's choice of the arm $I_t \in [k]$ are sampled independently from $a$ and $p_t$, respectively; then this choice $I_t$ and the rewards of the chosen arm $g_{I_t,t} $  are revealed to the player.  Theorem 3.5 in \cite {bubeck_book} proved an $\Omega (\sqrt {kT})$ lower bound using the probabilistic method.  This proof is based on information theoretic tools, in particular Pinsker's inequality, and entails averaging  over random choices of the safe arm, which is distributed according to an i.i.d. Bernoulli distribution with mean $\frac{1}{2} +\epsilon$. The remaining risky arms  have the same mean $\frac{1}{2} -\epsilon$ for $\epsilon = \gamma \sqrt {k/T}$ where  $\gamma>0$ is fixed.\footnote{The earlier reference \cite{auer} originally proved a similar lower bound.}  In the foregoing reference, the authors noted that they are not aware of any other techniques to prove bandit lower bounds. The methods in our paper make progress towards developing new techniques to prove such bounds.\footnote{By references \cite{rodman, zaborskis} discussed earlier in this section, similarly to the optimal player in the  symmetric two-armed Bernoulli bandit, the optimal player is myopic when it faces the $k$-armed bandit distribution described in the paragraph accompanied by this footnote.}  


As noted previously  the pseudoregret represents  the expected number of times the inferior treatment (risky arm) is chosen while  the regret represents the expected difference between the outcomes of the better arm in hindsight and the outcomes of the sequence of treatments chosen by the player. Nevertheless, the only known lower bounds for regret in general bandit problems are given by the pseudoregret associated with the stochastic Bernoulli distributions described in the previous paragraph. Our methods make progress towards developing new PDE-based techniques to prove lower bounds with respect to regret directly.

Another classic online learning problem is prediction with expert advice. This setting is rather different from the  bandit problem: the rewards of \textit{all} ``arms" (referred to as \emph{experts} in this setting) are revealed to the player in this problem, i.e., it is a \emph{complete information} problem. References \cite{Zhu, rokhlin, drenska2019prediction} connected this problem with a PDE, by considering a scaling limit as the number of time steps tends to infinity.  A little later, \cite{kobzar, kobzar_geom} obtained closely related results by more elementary Taylor-expansion-based methods. PDE-based analysis of regret has been used to determine asymptotically optimal strategies and regret in prediction with expert advice explicitly in certain cases \cite{bayraktar2019b, bayraktar2019a}, to analyze  variations of this classic problem  \cite{bayraktar2020prediction, drenska2020c, drenska2020b, drenska2020a, harvey2020optimal}, and to study drifting games \cite{wang22} and unconstrained online linear optimization \cite{zhang22}.   In reference \cite{bayraktar2022}, PDE-based methods connected with the prediction with expert advice literature were used to guarantee $O(\sqrt {T})$  regret in a bandit-like game where  the adversary's distribution in each round is revealed to the player in addition to the sampled gains. Notwithstanding the fundamental differences between stochastic bandits and  complete information problems, like prediction with expert advice,  the estimation of the value of the discrete game by a PDE solution using backwards induction (the``verification argument") in this paper is similar to that in \cite{kobzar}. 

The paper is organized as follows:  \cref{sec:main} sets forth our main results, \cref{sec: existing bounds} describes their relationship to the existing bounds, and the conclusion follows in \cref{sec:conclusions}. 

\section{Notation}

If $u$ is a function of several variables, subscripts denote partial derivatives (so $u_x$ and $u_t$ are first derivatives, and $u_{xx}$, $u_{xt}$ and $u_{tt}$ are second derivatives). In other settings, the subscript $t$ is an index; in particular, the arms' rewards and the player's choice of the arm at time $t$ are $g_t$ and $I_t$, and $g_{i,t}$ refers to the $i$-th component of $g_t$.  When no confusion will result, we sometimes omit the index $t$, writing for example $x$ rather than $x_t$; in such a setting, $x_{i}$ refers to the $i$-th component of $x_t$. 

If $u$ is a function, $\Delta u := \sum_i \frac{\partial^2 u}{\partial x_i^2}$ is its Laplacian; however, the standalone symbol $\Delta_k$ refers to the set of probability distributions on $\{1, ..., k\}$. $[k]$ and $[-T]$ denote the sets $\{1, ..., k\}$ and  $\{-T, ..., -1\}$ respectively for natural numbers $k$ and $T$.  $\mathbbm 1$ is a vector in $\mathbb R^k$ with all components equal to 1,  but $\mathbbm 1_S$ refers to the indicator function of the set $S$.  
If $f$ and $g$ are functions, $f*g$ represents their convolution. 

\section{Main results}
\label{sec:main}
\subsection {Optimality of the myopic player} In this section, we show that a myopic player is minimax optimal for the symmetric two-armed Bernoulli bandit.  
\label{sec:optimality}

In order to reduce the number of state variables, \emph{we center and normalize the range of rewards}, such that each arm will have the reward $-1$ with the probability of reward $0$ in the original problem, i.e., the rewards in the new game are given by 
\begin{align} \label{eq:centering}
\hat g_\tau = 2 g_\tau-\mathbbm 1.
\end{align}  
As shown in \cref{app:optimal_player}, this centering eliminates the need to track $s_1$ and $s_2$, the number of times each arm was pulled. \emph{In the remainder of this paper, we will only use the centered rewards but we will omit the superscript $\hat~$ (hat).}  We may also omit the word centered when we refer  the symmetric two-armed Bernoulli bandit with the centered rewards. 

Let the \emph{difference between the cumulative revealed rewards} be 
\begin{align}
 \xi_t^r :=  \sum_{\tau<t}g_{1,\tau} \mathbbm 1_ {I_\tau =1}  -g_{2,\tau}  \mathbbm 1_{I_\tau =2}. \label{eq:xir_noncentered}
\end{align}
Then the \emph{myopic player} $p^m$ is given as follows.
\begin{tcolorbox} \emph{Myopic player $p^m$} for the \emph{centered} symmetric two-armed Bernoulli bandit is
\begin{align}
\label{eq:optimal_player}
p^m(\xi_t^r) = 
\begin{cases}
 (1,0) ~\text {if}~\xi^r_t >0   \\
    \left(\frac{1}{2},\frac{1}{2}\right) ~\text {if}~\xi_t^r=0\\
        (0,1) ~\text {if}~\xi_t^r<0
        \end {cases} 
\end{align}   
\end{tcolorbox} 
This player $p^m$ chooses the safe arm such that the revealed rewards are most probable, i.e., it  is the maximum likelihood estimator of the safe arm (as we explain in the opening paragraphs of \cref{app:optimal_player}). We  show in the same appendix that this strategy is also \emph{minimax} optimal with respect to both regret and pseudoregret for the symmetric two-armed Bernoulli bandit.  
\begin{lemma} 
\label {lemma:optimal_player}
The player $p^m$ given by \cref{eq:optimal_player} is a minimizer of  \cref{eq:minimax_regret} and  \cref{eq:minimax_pseudoregret}. Moreover, this strategy makes the player indifferent about which arm is risky, that is, $R_T(p^m,a(1)) = R_T(p^m,a(2))$, and $\bar R_T(p^m,a(1)) = \bar R_T(p^m,a(2))$. 
\end{lemma}

\subsection{Centered state variables}
\label{sec:regret} In this section, we define the state variables used in the remainder of this work.  By \cref {lemma:optimal_player}, the minimax regret \cref{eq:minimax_regret} is $R^*_T= R_T (p^m, a(1)) =  R_T (p^m, a(2))$, i.e., the ``adversary''  achieves the maximum regret by making either arm safe.\footnote{Note that the player of course does not need to know which arm is safe in order to implement $p^m$.} \emph{Therefore, we will assume that the safe and risky arms are secretly labeled as arms 1 and 2, respectively, and we will omit the parameter $j$:} the distribution of the symmetric two-armed Bernoulli bandit will be denoted $a =(a_1, a_2)$  where $a_1$ is the distribution of the safe arm and $a_2$ is the distribution of the risky arm.  

We now define the  \emph{centered} difference between the cumulative revealed rewards as 
\begin{align}
 \hat {\xi}_t^r :=  \xi_t^r - \eps t. 
\label{eq:xir_centered}
\end{align}
As we will see below, this centering ensures that the increments of  $\hat {\xi}_t^r$ have  mean zero as this state variable evolves in accordance with the rule of our bandit problem. Using the centered variables will simplify the calculations in the remainder of the paper.  \emph{Accordingly, going forward we will only use the centered variable and omit the superscript $\hat~$ (hat).}  

After centering $\xi_t^r$,  the \emph{myopic player} $p^m$ is given as follows.
\begin{tcolorbox} \emph{Myopic player $p^m$} using the \emph{centered} $\xi^r_t$  is
\begin{align}
\label{eq:optimal_player_centered}
p^m(\xi_t^r,t) = 
\begin{cases}
 (1,0) ~\text {if}~\xi^r_t + \eps t >0   \\
    \left(\frac{1}{2},\frac{1}{2}\right) ~\text {if}~\xi_t^r+\eps t=0\\
        (0,1) ~\text {if}~\xi_t^r+\eps t <0
        \end {cases} 
\end{align}   
\end{tcolorbox} 
Let us also denote the \emph{centered difference between the cumulative hidden rewards} by
\[
\xi_t^h := \sum_{\tau<t}   g_{1,\tau}  \mathbbm 1_ {I_\tau =2}  - g_{2,\tau} \mathbbm 1_{I_\tau =1}- \eps t,
\]  
and define
\[
\xi_t := \left( \xi_t^h, \xi_t^r \right).
\]
Finally, let us consider the difference between the reward of the arm $J_\tau$ not chosen by the player and the arm $I_\tau$ chosen, that is 
\[
g_{J_{\tau},\tau} - g_{I_{\tau},\tau} = g_{1,\tau}   + g_{2,\tau}- 2 g_{I_{\tau},\tau}.
\]
We denote by $\eta$, the cumulative sum of these differences at time $t$:  
\[
\eta_t := \sum_{\tau <t}  g_{1,\tau}   + g_{2,\tau}- 2 g_{I_{\tau},\tau}.
\]
We will omit the subscript $t$ from the state variables defined above for simplicity whenever this information is clear from context.  

A brief calculation reveals that
\[
\max_{i} x_{i,0} = \frac{1}{2}\left(x_{1,0}+x_{2,0}+ |x_{1,0} - x_{2,0} | \right)= \frac{1}{2}\left (\eta_0 +|\xi_0^r + \xi_0^h | \right).
\] It is therefore natural to define
\[
\mu(\eta, \xi):= \frac{1}{2}\left (\eta +|\xi^r + \xi^h | \right). 
\]

\subsection{Asymptotically optimal regret using a $C^1$ PDE solution} 
\label{sec:opt_regret}

 Let  $v(\eta,\xi, t)$  represent the final-time regret if the bandit game starts at time $t$ with specified values of $\eta$ and $\xi$, and the player uses the $p^m$ strategy. Accordingly, for a symmetric two-armed Bernoulli bandit $a =(a_1, a_2)$ and an optimal player $p^m\equiv p^m(\xi^r_{t}, t)$, 
 \begin{align}
 &v(\eta_t, \xi_t, t)=   \mathbb E_{\subalign {&I_{t} \sim p^m\\ &g_{t} \sim a}} \Big [ \mathbb E_{\subalign {&I_{t+1} \sim p^m\\ &g_{t+1} \sim a}} \Big [ \dotso \Big [\mathbb E_{\subalign {&I_{-1} \sim p^m\\ &g_{-1} \sim a}} \mu ( \eta_t + \sum_{ \tau = t}^{-1} d \eta_\tau, \xi_t +  \sum_{ \tau = t}^{-1}  d\xi_\tau ) \Big ] \dotso \Big ] \Big ]\label{eq:v_minmax}
\end{align}
where in accordance with the information flow of bandit problem, at time $t$, $p^m$ is evaluated at $\xi^r_{t}$; at time $t+1$, $p^m$ is evaluated at $\xi^r_{t+1}$ etc.. The increments of the state variables are $d \eta_\tau =g_{1, \tau} +g_{2, \tau}-2g_{I_\tau}$  and 
\[
d\xi_\tau = (g_{1, \tau}  \mathbbm 1_ {I_\tau =2}  - g_{2, \tau} \mathbbm 1_{I_\tau =1}-\eps, ~g_{1, \tau} \mathbbm 1_ {I_\tau =1}  -g_{2, \tau}  \mathbbm 1_{I_\tau =2} -\eps).
\] Thus, the minimax optimal regret is
\begin{align}
 R_T^* = R_T(p^m,a) =  v(0, \eps T \mathbbm 1, -T) \label{eq:regret_value}.
\end{align}

According to the rules of the Bernoulli bandit problem, the domain of $v$ is restricted to the values of $\eta$, $\xi^h$, $\xi^r$, such that $\eta$, $\xi^h+\eps t$, $\xi^r +\eps t$ are integers, and $t \in [-T]$. This function $v$ is characterized iteratively:
\begin{subequations}{\label{eq:w_dp}}
\begin{align} 
v(\eta, \xi, 0) &= \mu(\eta, \xi) \label{eq:w_dp_a}\\
v(\eta, \xi, t) &=   \mathbb E_{a, p^m} ~ v(\eta+d \eta, \xi+d\xi, t+1)  ~\text{for} ~t \leq -1 \label{eq:w_dp_b}.
\end{align}
\end{subequations}
The foregoing characterization of $v$ resembles a numerical scheme for solving a PDE.\footnote {Our use of an iterative scheme is similar to that in \cite{kobzar}.} The essence of our analysis is that we identify the PDE and use it to estimate the regret.  We shall show that the leading order behavior of $v$ is given by a family of solutions $u$ of the following linear heat equation with a discontinuous source term: 
 \begin{subequations}{\label{eq:value_pde}}
 \begin{align} 
 &u_{t} +Lu = q \label{eq:value_pde1}\\
 &u(\eta, \xi,0) =\mu(\eta, \xi)
 \end{align} 
  \end{subequations}
where the spatial operator is just a Laplacian in $\xi$
\begin{align*}
Lu :=  \frac{\kappa}{2}\Delta_{\xi}  u  ~\text{and}~ \kappa = 1-  \epsilon^2
\end{align*}
 and the source term is
\begin{align*}
 q(\xi^r,t)  = 
 \begin{cases}
  -  \epsilon   &\text{if}~\xi^r +\eps t<0 \\    
\epsilon &\text{if}~\xi^r +\eps t > 0
\end{cases}.
\end{align*}

The form of the PDE \eqref{eq:value_pde} comes, roughly speaking, from the condition that  the definition \eqref{eq:w_dp} of $v$ should be a consistent numerical scheme for the PDE. The argument that this leads to \eqref{eq:value_pde} is the essence of what we do in \cref{app:lb} (though we work harder in the Appendix than would have been needed to find the PDE, since the Appendix also provides error estimates). 

The function $u$ can be determined explicitly. Let $y =\xi^r +\eps t$.  Then the function $\varphi$ of $y$ that solves the  following ODE 
\begin{align}\label {eq:ode}
\epsilon \varphi'  +\frac{\sigma}{2}\varphi''  = q.
\end{align}
where $\sigma$ will be fixed later. We require $\varphi$ to be smooth except at $0$, continuous at $0$, and to have at most linear growth at infinity. These conditions determine it up two constants: an additive constant, and the discontinuity (if any) of $\varphi'$ at $0$. We eliminate the former by always taking $\varphi(0) = 0$, and we do not eliminate the latter since our best result will be obtained when $\varphi'$ has a small ($\epsilon$-dependent) discontinuity at $0$. Accordingly, 
\begin{align}
\varphi(y) = |y|+
\begin{cases}
0 & \text{if}~ y\leq 0 \\
 b e^{-2 \frac{\epsilon}{\sigma} y}-b  & \text{if}~ y >0  
 \end{cases} \label{eq:varphi},
 \end{align}
 where the constant $b$ parametrizes the discontinuity of $\varphi'$ at $0$.  In this paper, we will assume that $b$ is positive (it will be in fact either $\kappa/\eps$ or close to it since $\varphi$ needs to be either $C^1$ or nearly so.) 
 
If we take 
\[
\sigma = \kappa,
\]
 then for $w(\eta, \xi, t) = u(\eta, \xi, t) - \varphi(\xi^r +\eps t)$,  
\begin{subequations}{\label{eq:homogen_pde}}
\begin{align}
&w_t+Lw =0 \label{eq:homogen_pde1}\\
&w(\eta, \xi, 0)  = \psi (\eta, \xi) \label{eq:homogen_pde2}
\end{align}
\end{subequations}
where $\psi (\eta, \xi) = \mu(\eta, \xi) - \varphi (\xi^r)$.  Therefore, the solution $u$ of  \cref{eq:value_pde} can be represented as 
\[
u(\eta, \xi, t) =   w(\eta, \xi, t) +\varphi(\xi^r +\eps t)= (\Phi*\psi) (\eta, \xi, t) +\varphi(\xi^r +\eps t) =u^h(\eta, \xi, t) + u^n(\xi^r, t)
\]
where 
\[
u^h(\eta, \xi, t)= (\Phi* \mu) (\eta, \xi, t), 
\]
which we will refer to as the \emph{homogeneous solution},  
\[
u^n(\xi^r, t) = \varphi(\xi^r +\eps t) -\hat \varphi(\xi^r, t),
\]
which we will  refer to as the \emph{non-homogeneous solution}, and 
\[
\hat \varphi (\xi, t)=  (\Phi* \varphi) (\xi, t).
\]
where the convolutions are in the $\xi$ variables only and $\Phi$ is the fundamental solution of \cref{eq:homogen_pde1}.  In \cref{app:pde_soln}, we show that after a suitable change of variables the above convolutions are one dimensional, and $\Phi$ reduces to the fundamental solution of the 1D heat equation in \cref{eq:phi}.
\begin{lemma}
\label{lemma:pde_soln}
A family of continuous solutions of \eqref{eq:value_pde} on $\R^3 \times (-\infty,0)$ with at most linear growth at infinity are given by 
\[
u(\eta, \xi, t) = u^h(\eta, \xi^r + \xi^h, t)+ u^n( \xi^r, t),
\] 
where 
\begin{align}
&u^h(\eta, z, t) = \frac{1}{2} \Big( \eta +\int_{\R} \Phi(z - s,2t)  | s | ds \Big), ~\Phi(s, t) = \frac{1}{\sqrt {-2 \pi \kappa t}}\exp \left(\frac{s^2}{2\kappa t}\right) ;  \label{eq:phi}\\
&u^n( \xi^r, t) = \varphi(\xi^r + \eps t) - \hat \varphi(\xi^r,t),~     \hat \varphi(\xi^r,t) =\int_{\R}  \Phi(\xi^r-s, t) \varphi (s) ds  ~\text{and}~ \sigma = \kappa, \label{eq:hatphi}
\end{align}
and the scalar $b$ in \eqref{eq:varphi} parametrizes this family. 
\end {lemma}
Note that the discontinuity of $\varphi'$ and therefore $u_{\xi^r}$ at $\xi^r + \eps t= 0$ is
\[
u^+_{\xi^r} - u^-_{\xi^r} =  \varphi'^{+} - \varphi'^{-}= 2(1 -  \epsilon b/\sigma)
\] where the superscripts$~^+$ and$~^-$ denote the right and left derivatives, respectively, at that point. Also  the discontinuity of $u_t$ at $\xi^r + \eps t= 0$ is
\[
u^+_t - u^-_t =  \eps(\varphi'^{+} - \varphi'^{-})= 2\eps(1 -  \epsilon b/\sigma).
\]  
Therefore, if $\sigma = \kappa$ and $b ={\sigma}/ \epsilon={\kappa}/ \epsilon$, then $u$ is the unique $C^1$ solution of \eqref{eq:value_pde}. For all $b$,  the discontinuity of $\varphi''$ and therefore $u_{\xi^r\xi^r}$ at $\xi^r + \eps t= 0$ is
\[
u^+_{\xi^r\xi^r} - u^-_{\xi^r\xi^r} =\varphi''^{+} - \varphi''^{-}= 4 (\epsilon /\sigma)^2 b,
\] and $u$ is $C^\infty$ for all $\xi^r + \eps t \neq 0$ and $t<0$.

In \cref{app:lb}, we prove, using induction backward in time, that the function $u$ approximates $v$ associated with the bandit problem up to a higher order ``error" term $E_1(t)$, which can be estimated explicitly.\footnote{While we use the asymptotic notation for conciseness and clarity of exposition, this and other error terms in this paper can be estimated by our methods with explicit constant prefactors.} To obtain this estimate, we need certain bounds on derivatives of $u$.  The following bounds are proved in \cref{app:error}.
\begin{lemma} \label{lemma:error}
 We have $\varphi'     = O(1 + b  \epsilon/\sigma)$,  and for $d\geq 2$, $\varphi^{(d)}  = O (b (\epsilon/\sigma)^d)$. For integer $d \geq 1$, 
 \[
 \partial^d_z u^h = O \Big( |\kappa t  |^{\frac{1-d}{2}} \Big),~ u^h_{tt}=  O\Big (\sqrt {\kappa} |t|^{-\frac{3}{2}}  \Big).
 \] 
 and
\begin{align*}
\partial^d_{\xi^r} \hat \varphi     =O \left( (1 + b \epsilon / \sigma)  |\kappa t |^{\frac{1-d}{2}}\right),~ \hat \varphi_{tt}     =O \left( (1 + b  \epsilon/\sigma) \sqrt {\kappa} | t |^{-\frac{3}{2}}\right). 
\end{align*}
 However, if $b =\sigma/ \epsilon$, i.e., $\varphi$ is $C^1$,  for $d \geq 2$,
\begin{align*}
\partial^d_{\xi^r} \hat \varphi     =O \left( \min\Big( \frac{ \epsilon}{\sigma},  | t |^{-\frac{1}{2}} \Big)|\kappa t |^{1-\frac{d}{2}} \right)~ \text{and}~ \hat \varphi_{tt}     =O \left( \min\Big( \frac{\epsilon}{\sigma}, | t |^{-\frac{1}{2}} \Big)| \kappa t |^{-1} \right). 
\end{align*} 
In all cases above, the bounds hold uniformly in $\xi$ and $\eta$. At $y=0$,  the above mentioned  bounds on $\varphi^{(d)}(0)$ apply to the  right derivatives (the left second and higher order derivatives are zero).
\end{lemma}

Our proof that $u$ approximates $v$ must address the following technical issue: even if $b =\sigma/ \epsilon$, so that $u$ is $C^1$, the second derivative of $u$ with respect to $\xi^r$ is discontinuous at $\xi^r+\eps t=0$ (due to the discontinuity of the source term $q$).  Therefore, when we use a third order Taylor polynomial to estimate how $u$ changes when $\xi^r$ evolves, the conditions of the Taylor theorem are not satisfied on any interval containing the discontinuity.  However,  according to the rules of the Bernoulli bandit problem, the domain of $v$ is restricted to integer values of $\xi^r+\eps t$. Therefore, we only need to bound the evolution of $u$ over integer $\xi^r+\eps t$'s. Near a point where $\xi^r+\eps t=0$ we will estimate the evolution of $u$ by taking advantage of the explicit form of $\varphi$. 

When $u$ is $C^1$, the above-mentioned discontinuity of $u_{\xi^r \xi^r}$  is a jump of size $O(\epsilon)$, but averaging leads to an ``error term" $O(\epsilon^2)$ at each time step.   Accordingly, over the $T$ periods, these errors contribute an $ O( \epsilon^2 T)$ error to $E_1(-T)$. Therefore, $u$  represents the leading order term of the regret only if $u$ dominates the error, i.e., $\lim_{T \rightarrow \infty}  \epsilon^2 T/u(0,\eps T \mathbbm 1, -T) = 0$ where $\epsilon$ depends on $T$. We shall show that this occurs in several regimes: 
\begin{itemize}
\item  \emph {small gap}  when $\epsilon = o(T^{-1/2})$;
\item  \emph{medium gap}  when $\epsilon = \gamma T^{-1/2}$ for constant $\gamma >0$; and
\item  \emph{large gap}  when  $\epsilon \rightarrow 0$  slower than a constant multiple of $ T^{-1/2}$.  
\end{itemize}
These results follow from the following theorem, which is proved in \cref {app:lb}, combined with \cref{thm:lb2}, which improves upon  \cref{thm:lb} in the large gap regime. 
\begin{theorem} 
 \label{thm:lb}
 Let  the functions $u$ and $v$ be as defined above, where $u$ is $C^1$ of  \eqref{eq:value_pde}, i.e.,  $\sigma = \kappa$ and $b = \sigma/\eps=\kappa/\eps$.  Then,
  \[
  |u(0,\eps T \mathbbm 1, -T)-v(0,\eps T \mathbbm 1, -T)| \leq E_1(-T) 
  \]
 where the error term  $E_1(-T)$ is  $O \Big(1+ \sqrt {\kappa} +  \Big (\frac{\kappa}{\eps} \nu (\eps)+   \frac{\kappa}{\eps} \rho (\eps)+1\Big)T \Big)$ and
\begin{align*}
&\rho (\eps) = \Big(   \exp \Big(-\frac{2\eps}{\kappa} \Big)-1\Big) (1+\eps)/2 \\
&\nu(\eps) =  \Big(\exp\Big (-\frac{2\eps}{\kappa} \Big)(1+\eps)/2+\exp\Big (\frac{2\eps}{\kappa} \Big)(1-\eps)/2 -1\Big)
\end{align*}
  When $\eps \rightarrow 0$, the leading order terms  of $ \frac{\kappa}{\eps}\nu$ and $ \frac{\kappa}{\eps} \rho$ are 
 \begin{align*}
& \frac{\kappa}{\eps} \nu(\eps)  \approx     2\eps^3 ~\text{and}~  \frac{\kappa}{\eps} \rho (\eps) +1 \approx \eps^2
 \end{align*}
  Therefore, $E_1(T)$ is $O \Big(1 + \eps^2 T \Big)$.
   \end{theorem}
 By \cref{eq:regret_value}, we have determined the regret up to the discretization error:
 \[
u(0,\eps T \mathbbm 1, -T)-E_1(-T) \leq  R_T(p^m,a) \leq u(0,\eps T \mathbbm 1, -T)+E_1(-T). 
\]
To analyze the regret in different gap regimes,  we examine the rescaled value of $u$ at the start of the game using the following result established in \cref{app:c}.
 \begin{corollary} \label{lemma:c}
 For \begin{align}
\label {eq:gamma}
 \gamma = \epsilon \sqrt{T},
 \end{align}
 if the leading order term of $b$ is $\frac{1}{\eps}$ as $\eps \rightarrow 0$, we have
  \begin{align}
c(\gamma)&:= \lim_{\eps \rightarrow 0} \frac {1}{\sqrt {T} } u (0,\eps T \mathbbm 1,-T) \nonumber \\
&=  \frac{1} {\sqrt {\pi}} e^ {-\gamma^2 } + \gamma \erf \left ( \gamma  \right) +\left(\frac{1}{\gamma} - \gamma \right) ~\erf  \left (\frac{\gamma}{\sqrt {2}} \right) - \sqrt {\frac{2}{  \pi}} e^{- \frac{\gamma^2}{2} }.   \label {eq:c}
 \end{align}
\end{corollary}
 In the small gap regime $\epsilon =o(T^{-1/2})$, and therefore $\gamma \rightarrow 0$ as $T \rightarrow  \infty$. Since $ \erf(x) \approx  \frac{2}{\sqrt {\pi}} (x - \frac{1}{3} x^3)$ near $0$, $\lim_{\gamma \rightarrow 0 } c(\gamma) = 1/\sqrt {\pi} $. This implies that the leading order regret $R_T(p^m,a)$ is $\sqrt { {T}/{\pi} } \approx .564 \sqrt {T}$, which matches the standard bound  obtained for this classic randomized adversary in the  setting of prediction with expert advice.\footnote{In this setting, the player strategy does not affect the leading order term of the regret. Therefore, the fact that the player does not have complete information in the bandit problem is irrelevant.  See Example 2 in \cite{kobzar} and note that the expectation of the maximum of two standard Gaussians is $ {1}/{\sqrt {\pi}}$.} 

In the medium gap regime,  $\epsilon = \gamma T^{-1/2}$ for constant $\gamma >0$.  Maximizing $ c(\gamma)$ numerically for $\gamma>0$ shows that for it has a unique maximizer $\gamma \approx .707$.  This yields the maximum leading order regret  $ \approx .572 \sqrt {T} $.  The function $c$ is plotted in \cref{fig:c_gamma}.

When  $\epsilon$ dominates  $T^{-1/2}$, i.e., $\gamma= \epsilon\sqrt {T} \rightarrow \infty$, which is denoted as $\epsilon= \omega \left(T^{-1/2}\right)$,  the above theorem allows to determine the leading order term of the regret  as long as  $\epsilon = o\left (T^{-1/3} \right)$. In this setting,  $\gamma \rightarrow \infty$ as  $T \rightarrow  \infty$.  Since $\erf(x) \approx  1 - e^{-x^2}/(x\sqrt {\pi})$ at infinity, $\lim_{\gamma \rightarrow \infty } \gamma c(\gamma) = 1$. Therefore, the leading order term of $u$ is $1/\epsilon$, which dominates $E_1(-T)$ given above as long as  $\epsilon = o\left (T^{-1/3} \right)$. 

If  $\epsilon$ approaches zero as a constant multiple of  $T^{-1/3}$ or slower, and $u$ is a $C^1$ function, \cref{thm:lb} does not recover the leading order term of the regret. In this regime, the leading order behavior of $u$ is still $1/\epsilon $, but it no longer dominates the $O(\epsilon^2 T) $ error. However, as shown in \cref{app:lb2} by selecting the suitable constant $b$  and making $\varphi'$ discontinuous at $\xi^r+\eps t=0$, we can offset  the  $O(\epsilon^2)$ error attributable to the discontinuity of $\varphi''$ at $\xi^r+\eps t=0$, and obtain the improved error term $E_0(-T)$. We will also set the prefactor $\sigma$ of the second order term in \eqref{eq:ode} to be different from the diffusion constant $\kappa$, which will reduce the error at $\xi^r+\eps t> 0$.

\subsection{Improved regret estimate in the large gap regime using a $C^0$ function} 
\label{sec:opt_regret}

In this section, we will use a modified version of the function $u$  reduce the discretization error in the large gap regime. Specifically,  by selecting the suitable constant $b$  and making $\varphi'$ discontinuous at $\xi^r+\eps t=0$, we can offset  the  $O(\epsilon^2)$ error attributable to the discontinuity of $\varphi''$ at $\xi^r+\eps t=0$.  Also by selecting a suitable prefactor $\sigma$ of the second order term in $\varphi$ we can \emph{eliminate} the discretization error attributable  to $\varphi$ for  $\xi^r+\eps t> 0$.  

Our  function $u$ used to estimate the regret will still be  represented as 
\[
u(\eta, \xi, t) =   u^h(\eta, \xi^h+\xi^r, t) + u^n(\xi^r, t)
\]
where the smooth PDE solutions are given by
\[
u^h(\eta, \xi^h+\xi^r, t)= (\Phi* \mu) (\eta, \xi^h+\xi^r, t), 
\]
and 
\[
\hat \varphi (\xi^r, t)=  (\Phi* \varphi) (\xi^r, t).
\]
where  $\Phi$ is still the fundamental solution of the heat equation given in \eqref{eq:phi} and $\varphi$ is given by \eqref{eq:varphi}. These properties will be sufficient to obtain the leading order regret estimates even though  $\varphi'$ has now a small ($\epsilon$-dependent) discontinuity at $0$ and, since $\sigma \neq \kappa$,
\[
u^n(\xi^r, t) = \varphi(\xi^r +\eps t) -\hat \varphi(\xi^r, t),
\]
and  $u$ are \emph{no longer}  solutions of linear heat equations. The bounds in \cref{lemma:error} will still apply, and the foregoing modifications will lead to the improved error term $E_0(-T)$, as shown in \cref{app:lb2}.
\begin{theorem} 
 \label{thm:lb2}
 Let  the functions $u$ and $v$ be as defined above, where  $u^h$ and $u^n$ are given by \cref{eq:phi} and \cref{eq:hatphi} respectively, with
 \begin{align}
\sigma = 2\eps \Big / \log  \Big(\frac{1+\eps}{1-\eps}\Big) \label{eq:sigma}
\end{align}
and $b =  \frac{1}{\eps}$.  Then  
  \[
  |u(0,\eps T \mathbbm 1, -T)-v(0,\eps T \mathbbm 1, -T)| \leq E_0(-T) 
  \]
 where the error term  $E_0(-T)$ is $O \left( 1+ (1+\frac{1}{\sigma}) \sqrt {\kappa}   \right)$.  As $\eps \rightarrow 0$, the leading order term of $E_0$ is $O \left( 1 \right)$.
 \end{theorem}
The preceding theorem improves upon \cref{thm:lb} and recovers the leading order term of the regret as long as $\epsilon$ approaches zero as \emph{any} rate.  The leading order term of the rescaled value of $u$ at the start of the game will be unchanged, as established in \cref{app:c}.
 \begin{corollary}
 \label{lemma:c_sigma}
 For $\gamma$ given by \eqref {eq:gamma}, if $\sigma$ is given by \eqref{eq:sigma}
and $b=\frac{1}{\eps}$, we have
  \begin{align}
c(\gamma)&:= \lim_{\eps \rightarrow 0} \frac {1}{\sqrt {T} } u (0,\eps T \mathbbm 1,-T) = \eqref {eq:c}
 \end{align}
\end{corollary}
The foregoing results are summarized in \cref{tab:summary}. If  $\epsilon$ is fixed as $T \rightarrow \infty$, our methods do not extract the leading order term of the regret: the $1/\epsilon $ leading order term of $ u$  will no longer dominate the $O(1)$ error. 

 \subsection{Asymptotically optimal pseudoregret}
 \label{sec:pseudo_regret}
 For a symmetric two-armed Bernoulli bandit $a$, the pseudoregret \cref{eq:pseudoregret} simplifies to
\[
\bar R( p,a)= \mathbbm E_{p,a} 2\eps s_2 
\]
where  $2\eps$ is the gap between the arms and $s_2$ is the number of time arm 2 (the risky arm) is pulled. 
  
Let $\bar v(\xi^r, s_2, t)$ represent the final-time pseudoregret if the bandit game starts at time $t$ with specified $\xi^r$ and $s_2$, and the player uses the strategy $p^m$.  This function $\bar v$ can be expressed similarly to \cref {eq:v_minmax} and is also characterized iteratively: 
\begin{subequations}{\label{eq:bar_w_dp}}
\begin{align} 
\bar v(\xi^r, s_2, 0) &= 2\eps s_2  \label{eq:bar_w_dp_a}\\
\bar v(\xi^r, s_2, t) &=   \mathbb E_{a, p^m} ~ \bar v(\xi^r+d\xi^r, s_2+ds_2, t+1)  ~\text{for} ~t \leq -1 \label{eq:bar_w_dp_b}
\end{align}
\end{subequations}
where $d\xi^r =  g_{1} \mathbbm 1_ {I =1}  -g_{2}  \mathbbm 1_{I =2} - \eps$ and  $ds_2 =   \mathbbm 1_{I =2}$.  The foregoing also resembles a numerical scheme for solving a PDE, similar to the one we considered in the previous section.  Again, the domain of $\bar v$ is restricted to integer values of $\xi^r$, $s_2$ and $t \in[-T]$, and we have  
\[
\bar R_T( p^m,a) = \bar v(\eps T,0, -T).
\] 
We identify the relevant PDE and use it to estimate the regret.  Specifically, we will show that the leading order behavior of $\bar v$ is given by a family of solutions $\bar u$ of the following linear heat equation with a discontinuous source: 
 \begin{subequations}{\label{eq:bar_value_pde}}
 \begin{align} 
 u_t+\frac{\kappa}{2}  u_{\xi^r \xi^r } = \bar q \label{eq:bar_value_pde1}\\
 u(\xi^r,s_2, 0)= 2 \eps s_2
 \end{align} 
  \end{subequations}
where the  source term is
\begin{align*}
\bar q(\xi^r)  = \begin{cases}
  -  2\epsilon   &\text{if}~\xi^r +\eps t <0 \\     
0 &\text{if}~\xi^r+\eps t  > 0
\end{cases}.
\end{align*}

Again, the form of the PDE \eqref{eq:bar_value_pde} comes, roughly speaking, from the condition that the definition \eqref{eq:bar_w_dp} of $\bar v$ should be a consistent numerical scheme for the PDE, and the argument that this leads to \eqref{eq:bar_value_pde}  parallels what we do in \cref{app:lb} to determine the PDE \eqref{eq:value_pde} in the context of regret (since the error estimates are somewhat different in the context of pseudoregret, they are determined in \cref{thm:bar_lb} and \cref {thm:lb2_bar}.)
  
Since the final value does not depend on $\xi^r$, the homogeneous solution that satisfies \cref{eq:bar_value_pde} without the source term is just the final value.  We let $\bar \varphi$ be a function of $y=\xi^r +\eps t $ satisfying 
\begin{align}\label {eq:bar_ode}
 \epsilon  \varphi'  +\frac{ \sigma}{2} \varphi''  = \bar q.
\end{align}
The $C^0$ solution of this ODE with $\bar \varphi(0) =0$ and at most linear growth at infinity is given by  
\begin{align}\label {eq:bar_ode_soln}
\bar \varphi(y) = 
\begin{cases}
- 2y & \text{if}~ y\leq 0 \\
 b  e^{-2 \frac{\epsilon}{\sigma} y}-b  & \text{if}~ y >0.  
\end{cases}
\end{align}
and when $\sigma = \kappa$, we obtain the following result. 
\begin{lemma}
\label{lemma:pde_soln_pseudo}
A family of continuous solutions of \eqref{eq:bar_value_pde} on $ \R^2 \times [-T,0)$ with at most linear growth at infinity are given by 
\[
\bar u( \xi^r, s_2, t) = 2 \epsilon s_2+ \bar u^n( \xi^r,  t)
\] 
where 
\begin{align}
 &\bar u^n( \xi^r, t) = \bar \varphi(\xi^r+\eps t) - \hat {\bar \varphi}(\xi^r,t)~ \text{and} ~  \hat {\bar \varphi}(\xi^r,t) =\int_{\R}  \Phi(\xi^r-s, t) \bar \varphi (s) ds  \label{eq:bar_uh}
\end{align}
where $\bar \varphi$ is given by \cref{eq:bar_ode_soln}, $\sigma = \kappa$, and   $\Phi(s, t) $ is given by  \cref{eq:phi} and $b$ is a constant that parametrizes the family of these solutions.   
\end {lemma}
If $b = \sigma / \eps = \kappa /\eps$, then $\bar u$ is the unique $C^1$ solution.  For other choices of $b$,  $\bar u$ is only $C^0$ at $\xi^r+\eps t =0$. For all $b$,  $\bar \varphi''$ and therefore $\bar u_{\xi^r\xi^r}$ have a jump at $\xi^r+\eps t= 0$.  Note that 
\[
\bar \varphi(y) = \varphi(y) - y
\]
where $\varphi$ is given by \cref{eq:varphi}, and thus
\[
 \hat {\bar \varphi}(\xi^r,t) =\int_{\R}  \Phi(\xi^r-s, t) (\varphi(s) - s) ds = \hat \varphi(\xi^r,t) -\xi^r
 \]
where    $\hat \varphi$ are given by  \cref{eq:varphi} and \cref{eq:hatphi} respectively.  Therefore,
\[
\bar u^n( \xi^r, t) = u^n( \xi^r, t) - \eps t 
\]
where   $u^n$ is given by  \cref{eq:hatphi}. Therefore, for $d \geq 2$, the  bounds on $\varphi^{(d)}$ and $\partial^d_{\xi^r} \hat \varphi$ in  \cref{lemma:error} apply to $ \bar \varphi^{(d)}$ and $\partial^d_{\xi^r} \hat {\bar \varphi}$ uniformly in $\xi^r$ and $s_2$. 

Since $u^n$ and therefore  $\bar u^n$ are smooth as $t \rightarrow 0$, we don't need to consider the final period separately for purposes of computing the discretization error. When  $u^n$ is $C^1$, the error accumulating in each time period attributable to  $u^n$ in \cref {eq:K}  is \begin{align} 
\bar K(t)&= O \left( \min \Big( \epsilon , \kappa | t |^{-\frac{1}{2}}\Big ) | t |^{-1}+\frac{\kappa}{\eps} \nu (\eps)+   \frac{\kappa}{\eps} \rho + 1\right) \nonumber 
\end{align}
The first term in the preceding expression is estimated by \cref{eq:min_estimate}, which leads to the following theorem.
\begin{theorem} 
 \label{thm:bar_lb}
 Let the functions $\bar u$ and $\bar v$ be as defined above, where $\bar u$ is the $C^1$ solution, i.e., $\sigma = \kappa$ and $b =  \sigma/\eps$. Then
  \[
  |\bar u(\eps T,0, -T)-\bar v(\eps T,0 -T)| \leq \bar E_1(-T) 
  \]
 where the error term  $\bar E_1(-T)$ is  
 \begin{align*}
 O \Big(  \epsilon \min ( \log (\kappa^2/ \eps ^2) +   1, \log T) +  \big(\frac{\kappa}{\eps} \nu (\eps)+   \frac{\kappa}{\eps} \rho + 1\big)T\Big).   
 \end{align*}
 When $\eps \rightarrow 0$,  the leading order terms  of $ \frac{\kappa}{\eps}\nu$ and $ \frac{\kappa}{\eps} \rho$ are set forth in  \cref{thm:lb}, and $\bar E_1(T)$ is 
\[
O \Big(\epsilon \min ( \log (1/ \eps ^2) +   1, \log T) + \eps^2 T)  \Big).
\]
\end{theorem}
 Since $\bar R_T( p^m,a) =\bar v(\eps T,0, -T)$, we have determined the pseudoregret up to the discretization error 
 \[
\bar u(\eps T,0, -T)-\bar E_1(-T) \leq  \bar R_T( p^m,a) \leq \bar u(\eps T,0, -T)+ \bar E_1(-T). 
\]
To analyze the pseudoregret in different gap regimes,  we examine the prefactor of the leading order term of $\bar u$ -- it is determined by taking the terms attributable to $u^n$ in  \cref{app:c} (plus $\eps T /\sqrt {T}$ which corresponds to adding $\gamma$ in  \cref {eq:bar_c} below).  
\begin{corollary}
 For $\gamma$ given by \cref{eq:gamma}, 
 if $\sigma=\kappa$ or $\sigma$ is given by \eqref{eq:sigma}, and the leading order term of $b$ is $\frac{1}{\eps}$, as $\eps \rightarrow 0$ we have
  \begin{align}
 \label {eq:bar_c}
\bar c(\gamma):=\frac {1}{\sqrt {T} } \bar u (\eps T,0,-T) =  \left (\frac{1} {\gamma} -\gamma \right)\erf  \left (\frac{\gamma }{\sqrt {2}} \right) -  \sqrt { \frac{2}{ \pi}} e^ {-\frac{\gamma^2}{2}} +\gamma.  
 \end{align}
\end{corollary}
 \begin{figure}[tbhp]
 \centering
   \includegraphics[width=.6\textwidth]{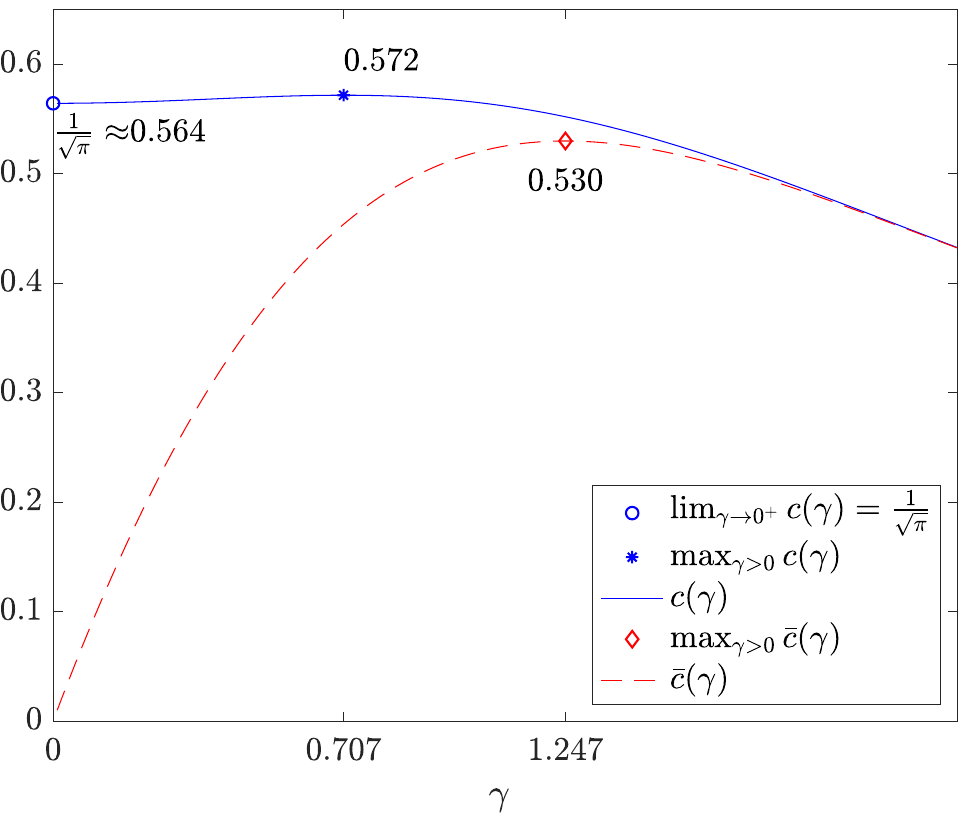} 
   \caption{Plots of the prefactors $c$ and $\bar c$ given by \cref{eq:c} and \cref{eq:bar_c}, respectively, of the leading order terms of optimal regret and pseudoregret as functions of $\gamma = \epsilon \sqrt{T}$ (medium gap regime).} 
   \label{fig:c_gamma}
\end{figure} 

 In the medium gap regime, this function $\bar c$ provides the constant prefactor of the leading order term of the regret, which is plotted in \cref{fig:c_gamma}. Maximizing  \cref{eq:bar_c} numerically for $\gamma>0$ shows that  it has a unique maximizer at $\gamma \approx 1.274$.  This yields the leading order regret  $ \approx .530\sqrt {T} $, which matches the result in \cite {bather}. \emph{This and other  references cited in this work use the 0/1 scaling of the rewards. Therefore, the constant prefactors of regret bounds in those references are smaller by a factor of $1/2$ than those in our paper.}

 In the small gap regime, since $\erf(x) \approx  \frac{2}{\sqrt {\pi}} (x - \frac{1}{3} x^3)$ near $0$, $\lim_{\gamma \rightarrow 0 }\bar c(\gamma)/\gamma =1  $. This yields  $\epsilon T$ as the leading order term of the pseudoregret.   

In the large gap regime, a computation similar to the corresponding computation in the previous section shows that the resulting leading order term of $\bar u$ is $1/\eps$. This term dominates $\bar E_1(-T)$ as long as $\eps \rightarrow 0$ as $T \rightarrow \infty$; so under this condition it reflects the leading order term of the pseudoregret. However, we can again reduce the first term of the error    $O(\epsilon^2T)$ to $0$ by making $\bar \varphi'$ discontinuous at $\xi^r+\eps t=0$. 

When $\sigma$ is given by \eqref{eq:sigma} and $b = 1/\eps$, the cumulative error  attributable to $u^n$ in \cref {app:lb2}  is 
\[
E_0(t) =O \left(  (1+\frac{1}{\sigma}) \sqrt {\kappa} \right).
\]
which leads to the following error estimate. 
\begin{theorem} 
 \label{thm:lb2_bar}
 Let  the functions $\bar u$ and $\bar v$ be as defined above, where $\bar u$ is a $C^0$ function with $\sigma$ given by \eqref{eq:sigma} and $b =  \frac{1}{\eps}$.  Then  
  \[
  |\bar u(\eps T,0, -T)-\bar v(\eps T,0,- T)| \leq \bar E_0(-T) 
  \]
 where the error term  $\bar E_0(-T)$ is $O \left( (1+\frac{1}{\sigma})  \sqrt {\kappa}   \right)$. When $\eps \rightarrow 0$, $\bar E_0(T)$ is $O \Big(1 \Big)$.
\end{theorem}
Using this improvement of  \cref{thm:bar_lb} in the large gap regime, we recover the leading order term of the pseudoregret as long as $\epsilon \rightarrow 0$ at any rate as $ T \rightarrow \infty$. The foregoing results are also summarized in \cref{tab:summary}. 

\begin{table}[tbhp]  
  \centering
  \begin{tabular}{|c|c|c|c|}
    \cline{2-4}
    \multicolumn{1}{c|}{} & \emph{Small gap}  &   \emph{Medium gap}  &  \emph{Large gap:} $\eps \in$  \\ %
     \multicolumn{1}{c|}{} & $\eps= o \big(T^{-\frac{1}{2}} \big)$ & $ \eps= \gamma T^{-\frac{1}{2}}$ & $ \big [\omega \big(T^{-\frac{1}{2}}\big),  o (1) \big]$   \\ \hline %
    $R_T(p^m, a)$  & $\frac{1}{\pi} T^{\frac{1}{2}} \approx .564 ~T^{\frac{1}{2}}$ & $c(\gamma)T^{\frac{1}{2}} (\text{max}~ .572 ~T^{\frac{1}{2}})$ & $ 1/\eps$ \\ %
    $\min(E_1(-T),E_0(-T))$  & $O(1)$ & $O(1)$ & $ O (1 )$ \\  \hline 
   $\bar R_T (p^m,a)$ & $\epsilon T$ & $\bar c(\gamma)T^{\frac{1}{2}} (\text{max}~ .530~ T^{\frac{1}{2}}) $ & $1/\eps$\\ 
    $\min (\bar E_1(-T),\bar E_0(-T))$  & $O(\epsilon \log T + \eps^2 T)  $ & $O 
    (1)$ & $O (1 )$ \\ \hline %
  \end{tabular}
  \caption{ The leading order terms of optimal regret and pseudoregret and discretization errors for the symmetric two-armed Bernoulli bandit. The maximum values of $c$ and $ \bar c$ for $\gamma>0$ in the medium gap regime are obtained by numerical optimization and rounded to 3 decimal places.} 
  \label{tab:summary}
\end{table}

If  $\epsilon$ is fixed as $T \rightarrow \infty$, our methods do not extract the leading order term of the pseudoregret: the $1/\epsilon $ leading order term of $\bar u$  will no longer dominate the $O(1)$ error.

\section {Relationship to existing results}
\label{sec: existing bounds}
As mentioned earlier, the symmetric two-armed Bernoulli bandit was previously considered in \cite{bather}. That paper determined an asymptotically optimal Baysian pseudoregret $0.530\sqrt{T}$, which matches our estimate. We are not aware of the leading order terms of the minimax optimal regret or pseudoregret having being determined previously  (as opposed to Bayesian pseudoregret) in the symmetric version of the problem. 

Since the regret and pseudoregret in the symmetric two-armed Bernoulli bandit bounds from below the minimax regret in the general two-armed stochastic and adversarial bandit problems, our results lead to an improved \emph{nonasymptotic} lower bounds for the latter classes of problems.\footnote{The existing \emph{asymptotic} lower bound for the general (non-symmetric) two-armed Bernouilli bandit is still sharper however than the leading order term lower bound that follows from our results.  In this setting, the minimax pseudoregret given by $M(T) = \min_p \max_{\mu_1, \mu_2} \bar R_T(p,a(\mu_1, \mu_2))$, where $\mu_1$ and $\mu_2$ are the means of the arms, is asymptotically bounded by
\[ 
0.612 \leq \lim_{T \rightarrow \infty} \inf  M(T) / T^{1/2} \leq \lim_{T \rightarrow \infty} \sup  M(T) / T^{1/2} \leq 0.752
 \]
where the lower and the upper bounds were determined in \cite{bather} and \cite{vogel1960} respectively.}  Previously, the best  nonasymptotic lower bound $\sqrt { 2T } /10  \approx .14 \sqrt{T}$ known to us for the general two-armed bandit problem is obtained for our symmetric Bernoulli distribution using information-theoretic tools.\footnote{See Theorem 3.5 in  \cite{bubeck_book}.   For further reference, in the general two-armed  bandit setting, the best nonasymptotic pseudoregret upper bound $2\sqrt { T \log {2}} \approx 1.665 \sqrt{T}$ is achieved by information-directed sampling (Specifically, Proposition 3 in \cite{russo14} established a $\sqrt{2 \log {\mathcal |A|} k T}$  pseudoregret bound for Bayesian bandits where  in the context of two-armed bandits the number of player's actions is $\mathcal |A|=2$.  Subsequently, Corollary 10 in  \cite{lattimore21} extended this bound to oblivious adversaries in the minimax setting.) The best nonasymptotic regret upper bound $(10.3 \sqrt {2 \log 2} + 2\sqrt {2/ \log 2})  \sqrt {T} \approx 15.525  \sqrt {T}$ known to us is achieved by an exponential weights-based algorithm (Theorem 3.4 in  \cite{bubeck_book}).}
 
  
As noted in \cref{sec:intro}, reference \cite{KW23} determined the upper bound on the pseudoregret of the diffusion limit of the Thompson sampling strategy in the general two-armed bandit setting in the large gap regime.  Specifically,  they showed that the rescaled pseudoregret ($\bar c = \bar R_T/\sqrt {T}$) guaranteed by Thompson sampling with respect to the rescaled gap $\gamma =  \eps \sqrt T $ is upper bounded as follows 
 \[
 \bar  c  \gamma ^\beta \rightarrow 0
 \] as the rescaled gap $\gamma =  \eps \sqrt T  \rightarrow \infty$ for any $\beta \in (0,1)$.  Our result that the leading order term of $\bar R_T(a,p)$ is $\frac{1}{\eps}$ implies that $\bar  c  \gamma \rightarrow 1$.  Therefore, for any  $\beta$, as above,  $\bar  c  \gamma ^\beta \rightarrow 0$. Therefore, the optimal regret in the two-armed symmetric bandit also satisfies the foregoing upper bound. This confirms that the optimal player performs in the two-armed symmetric setting no worse than (potentially suboptimal) Thompson sampling, and therefore our results are consistent with \cite{KW23}.

\section{Conclusion} 
\label{sec:conclusions}

In this work, we determine the minimax optimal player and characterize the asymptotically optimal minimax regret and pseudoregret of the symmetric two-armed Bernoulli bandit by explicit  solutions of linear heat equations when the gap between the means of the arms goes to zero as the number of prediction periods approaches infinity. We also provide new estimates of the non-asymptotic error.  Our PDE-based proof works despite the fact that the solution of our PDE has discontinuous derivatives and is not a classical one on the entire domain. Although  optimal player strategies are not known for more general bandit problems, we expect that the methods of this paper should be useful in considering how regret accumulates under specific player strategies, even when they are not known to be optimal. Separately, there are other bandit problems that do not require exploration, like  the symmetric two-armed bandit in the \emph{fixed} gap regime and the symmetric $k$-armed Bernoulli bandit distributions (which are, as discussed above, used to bound the regret from below in general $k$-armed bandit problems). We expect that the methods of our paper could be applied to such problems as well. 

\section*{Acknowledgements} 
V.A.K.  acknowledges helpful input from Chris Wiggins, and support from NSF grant DMS-1937254. R.V.K. acknowledges support from NSF grant  DMS-2009746.

\appendix

\section{Proof of \cref{lemma:optimal_player}} 
\label{app:optimal_player}

 A minimax optimal player $p^*$ for the regret minimization problem is, by definition, a
minimizer of \cref{eq:minimax_regret}, which can be expressed in the $\eta$ and $\xi$
coordinates as
\begin{align} \label{eq:minimax_regret_transformed}
\min_p \max_{j \in[2]} \mathbb E_{a(j),p} \Big [\frac{1}{2} (\eta_0 + |\xi_0^r + \xi_0^h|) \Big] .
\end{align}
Here, as discussed in \cref{sec:intro}, $p = (p_{-T}, \dotsc , p_{-1})$ ranges over all possible player
strategies; in particular, each $p_t$ depends only on the history that is available to the player at
time $t$. We shall show in this section that the strategy $p^m$ (defined by \cref{eq:optimal_player}) is
minimax optimal.

\emph{For purposes of this \cref{app:optimal_player}, we will use centered gains $g_i = \pm 1$ but consistently with \cref{sec:optimality} we will not center $\xi^r_t$ to have zero mean, i.e., we will use the definition of $\xi^r_t$ given by \cref{eq:xir_noncentered}.  (Elsewhere in the paper we will use centered $\xi^r_t$ given by \cref{eq:xir_centered}).} 

We start with an argument that makes this conclusion plausible (while also displaying
transparently some key ideas). Recall that in terms of the centered gains $g_i = \pm 1$,
$ p^m_t$ depends only on $\xi^r_t = G_{1,t} - G_{2,t}$, where at any time $t$ the observed gains are
$G_{i,t} = \sum_{\tau < t} g_{i,\tau} \1_{I_\tau = i}$. It chooses arm $1$ if $\xi^r_t > 0$, it chooses
arm $2$ if $\xi^r_t < 0$, and it chooses the two arms with probability $1/2$ each if $\xi^r_t = 0$. This is a
{\it maximum likelihood} estimator of the safe arm. Indeed, due to the symmetry of the two bandit arms, if $g$ is a trial from one arm then $-g$
can be viewed as a trial from the other arm. Using this observation to convert observed trials
of arm $2$ to trials of arm $1$, we see that the sample mean of the resulting gains of arm $1$ is
positive exactly when $\xi^r >0$. Thus: based on the sample means available at time $t$, arm $1$ is more likely to be
safe if $\xi^r_t > 0$, arm $2$ is more likely to be safe if $\xi^r_t < 0$, and no distinction is possible if
$\xi^r_t = 0$. Since the gains of the arms at distinct time steps are independent, the order in which the
arms were chosen should be irrelevant; and since sampling either arm gives statistical information about both
arms, the information gained at each step does not depend on the player's choices. Thus, the sample
means just discussed are the {\it only} information available to the player at time $t$. In view of
this, it is difficult to imagine how a different player strategy could do better than $p^m$.

But the preceding argument is not a proof. The rest of this section provides a rigorous argument. Our argument
is in a sense inductive. In fact, starting from any minimax optimal player strategy
$p^* = (p^*_{-T},\dotsc, p^*_{-1})$ that differs from $p^m$, we consider a new strategy
$p = (p_{-T},\dotsc, p_{-1})$ obtained as follows:

\begin{enumerate}
\item If $\tau$ is the earliest time such that
\[
p^*_\tau  \neq p^m
\]
we set
\[
p_t = p^m \mbox{ for } t \leq \tau .
\]
(This leaves $p_t$ unchanged relative to $p^*$ at times $t< \tau$, and changes it to $p^m$ at time $\tau$).

\item At subsequent times $t > \tau$ we choose $p_t$ so that it is {\it statistically equivalent} to $p_t^*$.
Rather than give a formula for $p_t$, it is more convenient to say how to sample it. For any given history
of player choices and observed gains $H_{t-1} = (I_{-T}, \dotsc , I_{t-1}; g_{I_{-T},-T}, \dotsc , g_{I_{t-1},t-1})$,
the player samples $p_t$ as follows:

\begin{itemize}
\item First, the player replaces $I_\tau$ by a choice $\tilde{I}_\tau$ sampled using $p^*_\tau$ (evaluated, of
course, at the given history $H_{\tau-1}$ through time $\tau - 1$).

\item If $\tilde{I}_\tau \neq I_\tau$ then $g_{\tilde{I}_\tau,\tau}$ has not been observed; however
the statistically equivalent quantity $-g_{I_\tau,\tau}$ {\it has} been observed. So the player samples
$p_t$ by sampling $p^*_t$ evaluated at the modified history $\tilde H_{t-1}$ obtained by not only changing $I_\tau$ as
indicated above but also replacing the time $\tau$ gain $g_{I_\tau,\tau}$ by
\[
\tilde{g}_{\tilde{I}_\tau, \tau} =
\begin{cases}
g_{I_\tau, \tau} & \text{if}~\tilde I_{\tau} = I_{\tau}\\
- g_{I_\tau, \tau} &\text{if} ~\tilde I_{\tau} \neq I_{\tau}
\end {cases}.
\]
\end{itemize}
Using this procedure, the player's choices (and therefore also her gains) at times $\tau + 1, \dotsc, -1$ are
statistically identical to those obtained using $p_t^*$.
\end{enumerate}

We shall show that the strategy $p$ just defined does at least as well as $p^*$. Iterating the preceding
argument finitely many times, it follows that the strategy $p^m$ is optimal, as claimed.

\subsection {Some simplifications and preliminary calculations}

We begin by giving an alternative characterization of a minimax optimal player: it is one that maximizes
the worst-case expected player gains:
\begin{equation} \label{eq:max-min-player-gains}
\max_p \min_j \mathbb E_{a(j), p} \sum_{t \in [-T]} g_{I_{t},t} .
\end{equation}
To explain why, we observe that the player's strategy $p$ and the adversary's choice $j$ can only
influence the value of $\eta_0$ in \cref{eq:minimax_regret_transformed}. This is because
$\xi_0^r + \xi_0^h$ does not depend on $p$, and only the sign of  $\xi_0^r + \xi_0^h$, as a random
variable, depends on $j$ -- so that the expectation of  $|\xi_0^r + \xi_0^h|$ does not depend on $j$
either. Thus, to solve \eqref{eq:minimax_regret_transformed} the player needs to find the optimal $p$ for
\[
\min_p \max_j \mathbb E_{a(j),p} \sum_{t \in [-T]} \frac{1}{2} \left(g_{1,t}   + g_{2,t}- 2 g_{I_{t},t} \right).
\]
Since $\mathbb E_{a(j),p} [g_{1,t} + g_{2,t}]=0$ for all $p$ and $j$, it suffices for the player to optimize
\begin{equation}
\min_p \max_j \mathbb E_{a(j), p} \sum_{t \in [-T]} - g_{I_{t},t} =
- \max_p \min_j \mathbb E_{a(j), p} \sum_{t \in [-T]} g_{I_{t},t}.
\end{equation}
This confirms the alternative characterization \eqref{eq:max-min-player-gains}.

Next, let us write the objective of \eqref{eq:max-min-player-gains} more explicitly. We have
\begin{align} \label{eq:expected-player-gains}
\mathbb E_{a(j), p} \sum_{t \in [-T]} g_{I_{t},t}=&~
\mathbb E_{g_{-T} \sim a(j)} \langle p_{-T}, g_{-T}  \rangle + \sum_{t \in [-T+1]} \nu_t
\end{align}
where
\begin{equation} \label{eq:defn-of-nu-t}
\nu_t = \E_{a(j), p} ~ g_{I_{t},t}
\end{equation}
can be written (remembering that $p$ depends on revealed history $H_{t-1}$, as defined in \cref{eq:history}) as
\begin{align}
\nu_t =& \sum_{H_{t-1}} \mathbb E_{g_t \sim a(j)}
\langle p_t, g_{t} \rangle \text{Prob}_{a(j), p} (H_{t-1} ) \nonumber\\
&= \sum_{H_{t-1}} \Big(\frac{1}{2} -\eps (-1)^{j}  \Big(p_{t,1}-\frac{1}{2}\Big) \Big )
\text{Prob}_{a(j), p} (H_{t-1}) \label{eq:expected-gain-explicit}
\end{align}
where we sum over all possible histories available at time $t$. Moreover, in accordance with \cref{eq:regret},
\begin{align*}
\text{Prob}_{a(j), p} (H_{t-1})
=& \kappa_{H_{t-1}} \pi_{j, H_{t-1}}
\end{align*}
with the convention that if $H_{t-1}$ is the specific history under discussion,
\begin{align*}
&\kappa_{H_{t-1}} =  \text{Prob}_{p_{-T}} (I_{-T} ) \text{Prob}_{p_{-T+1}} (I_{-T+1}|H_{-T}  ) \dotsm
\text{Prob}_{p_{t-1}} (I_{t-1}|H_{t-2})
\end{align*}
and
\begin{align*}
\pi_{j, H_{t-1}} &= \text{Prob}_{a(j)} (g_{I_{-T}, {-T}}) \text{Prob}_{a(j)} (g_{I_{-T+1}, -T+1}) \dotsm
\text{Prob}_{a(j)} (g_{I_{t-1}, t-1})\\
& = \text{Prob}_{a(j)} (g_{I_{-T:t-1}}).
\end{align*}
Note that $\kappa_{H_{t-1}}$ does not depend on $j$; this reflects the fact that the player's strategy
depends only on the history that was revealed to her (she does not know $j$).

We emphasize that 
$p_t$ is function of histories taking values in the space of probability distributions on the two arms. For example, given a strategy $p$ and history $H_{-T} =(I_{-T},g_{I_{-T},-T} )$ available after the first prediction round at time ${-T}$,
\[
\text{Prob}_{p_{-T+1}} (I_{-T+1}|H_{-T} )
\]  is the probability that this player chooses arm $I_{-T+1}$ at time $-T+1$ if at time $-T$ she chose arm $I_{-T}$ and received the gain $g_{I_{-T},-T}$.

The probability of a particular sequence of gains is easily made explicit. The calculation is simplest
when the gains are $0$ and $1$. For any list of revealed $0/1$ gains $g_{I_{-T:t-1}}$ at time $t$, let $s_i$
be the number of times arm $i$ was chosen, and let $G_i = \sum_{s < t} g_{i,s} \1_{I_s = i}$ be the sum
of the revealed gains from arm $i$. Then
\begin{align} \label {eq:prob_g}
\text{Prob}_{a(j)} (g_{I_{-T:t-1}})  =
\left(\frac{1+ \epsilon }{2} \right)^{G_j}\left(\frac{1- \epsilon}{2} \right)^{s_j -G_j}
\left(\frac{1 - \epsilon}{2} \right)^{G_m}\left (\frac{1+ \epsilon}{2} \right)^{s_m -G_m}
\end{align}
and $m=2$ if $j =1$ and $m=1$ if $m=2$.\emph{We will omit the subscript of $H$ when doing so is not expected  to cause confusion.}  Since \eqref{eq:prob_g} is, by definition, the value of
$\pi_{j,H}$, a little algebra reveals that
\begin{equation} \label{eq:ratio-of-probabilities}
\frac{\pi_{1,H}}{\pi_{2,H}} = \left( \frac{1+\epsilon}{1-\epsilon} \right)^{2G_1 - 2G_2 + s_2 - s_1}.
\end{equation}
 Evidently, $\pi_{1,H} > \pi_{2,H}$ exactly the exponent on the right is positive. Since $\pi_{j,H}$ is the probability of the given sequence of gains if arm $j$ is safe, we have confirmed
that $p^m$ chooses the arm that, by a maximum likelihood estimate, is more likely to be safe, given the
observed sequence.

Since we prefer to work with centered gains (taking the values $\pm 1$), let us put the
preceding calculation in those terms. To avoid confusion, for this paragraph (only) we denote the centered gains
by $\hat{g}_i$ (so $\hat{g}_i = 2 g_i -1$) and we write
$\hat{G}_i = \sum_{\tau < t} g_{i,\tau} \1_{I_\tau = i}$ for the analogue of $G$ using centered gains. Then
one easily checks that $\hat{G}_i = 2 G_i - s_i$, so that the exponent on the right side of
\eqref{eq:ratio-of-probabilities} is just $\hat{G}_1 - \hat{G}_2 = \xi^r$. This agrees, of course,
with our earlier argument that the sign of $\xi^r$ determines which arm is more likely to be safe, given
the observed gains. \emph{For the remainder of this appendix, we will continue to work with the
centered gains, but (as in the body of the paper) we shall write $g_i$ not $\hat{g}_i$ to avoid notational
clutter.}

\subsection{The optimality of $p^m$}

We are ready to explain the optimality of $p^m$. Recall the plan indicated earlier:
given an optimal strategy $p^*$, we consider the first time $\tau$ when it differs from $p^m$, and we consider the
alternative strategy (discussed earlier) that uses $p^m$ at time $\tau$ and is statistically equivalent to $p_t^*$
for $t > \tau$. Our goal is to show that the player's worst case expected gains \eqref{eq:max-min-player-gains}
are at least as large under the alternative strategy as under $p^*$.

Since the alternative strategy is statistically identical to $p^*$ at times other than $\tau$, we may focus exclusively
on the situation at time $\tau$.

The case $\tau = -T$ is simple but instructive. At the initial time there is no history and $\xi^r = 0$, so
$p_{-T}^* = (p^*_{-T,1},p^*_{-T,2})$ is just a probability distribution on the two arms and $p^m = (1/2,1/2)$.
When we restrict our attention to time $-T$, the max-min \eqref{eq:max-min-player-gains} becomes
$$
\max_{0 \leq p_{-T,1} \leq 1} \min_{j=1,2} \E_{g_{-T} \sim a(j)} \langle p_{-T}, g_{-T}  \rangle,
$$
which reduces by simple algebra to
$$
\max_{0 \leq p_{-T,1} \leq 1} \min_{j=1,2} \Big(\frac{1}{2} -\eps (-1)^{j} \Big(p_{-T,1}-\frac{1}{2}\Big)\Big).
$$
The optimal $p_{-T,1}$ is easily seen to be $1/2$ -- the value chosen by $p^m$; moreover, choosing this value
makes the player \emph{indifferent} to whether $j=1$ or $2$ (that is, the player is indifferent to the adversary's
choice which arm is safe).

For $\tau > -T$, the argument is similar in spirit though the details are more involved. We shall show that
among strategies satisfying $p_t = p^m$ for $-T \leq t < \tau$, the choice $p_{\tau} = p^m$ is optimal for
\begin{equation} \label{eq:max-min-of-nu-tau}
\max_{p} \min_{j=1,2} ~ \nu_\tau
\end{equation}
where $\nu_\tau$ is defined by \eqref{eq:defn-of-nu-t}; moreover, the proof will reveal that this choice
makes the player indifferent at time $\tau$ to whether $j=1$ or $j=2$.

The argument relies on grouping the histories in a convenient way. Given any history
$H_{\tau}= (I_{-T:\tau},g_{I_{-T:\tau}})$, we say $H^c_{\tau}= (J_{-T:\tau},g_{J_{-T:\tau}})$ is its complement if $H^c_{\tau}$ lists the same
gains but attributes them to the opposite arms; thus, for example, if $\tau = -T + 3$, the complement of
$H_{\tau} = (1,1,2,+,+,-)$ is $H^c_{\tau} = (2,2,1,+,+,-)$. \emph{We will also omit the subscript $H^c$ when doing is not expected to cause confusion.} 
Notice that every history has a complement, no history is its own complement, and if $H^c$ is the complement of $H$ then $H$ is the complement of $H^c$.  Given a complementary pair $H$ and $H^c$, we introduce the notation
\begin{align*}
&p_{H} := \text{Prob} (I_{\tau}=1 | H_{\tau-1})~\text{and}~
p_{H^c} := \text{Prob} (I_{\tau}=1 |  H^c_{\tau-1})
\end{align*}
and we introduce the analogues for $H^c$ of $\kappa_{H}$ and $\pi_{j,H}$,
\begin{align*}
&\kappa_{H^c} = \text{Prob}_{p_{-T}} (J_{-T} ) \text{Prob}_{p_{-T+1}} (J_{-T+1}|H^c_{-T}  ) \dotsm
\text{Prob}_{p_{\tau-1}} (J_{\tau-1}|H^c_{\tau-2}) \\
&\pi_{j, H^c}  = \text{Prob}_{a(j)} (g_{J_{-T:\tau-1}}).
\end{align*}
Since $\pi_{1, H}=\pi_{2, H^c} $ and $\pi_{2,H}=\pi_{1,H^c}$, it is convenient to group the
terms in $\nu_\tau$ as follows:
\begin{align}
 &(-1)^{j} \Big( \Big(p_H-\frac{1}{2}\Big) \kappa_H \pi_{j, H} +
\Big(p_{H^c}-\frac{1}{2}\Big) \kappa_{H^c} \pi_{j, H^c} \Big ) \nonumber \\
&=\begin{cases}
\Big(\frac{1}{2}-p_H\Big) \kappa_H \pi_{1, H} +
\Big(\frac{1}{2}-p_{H^c} \Big) \kappa_{H^c} \pi_{1,H^c} & \text{if}~ j=1 \\
\Big(p_H-\frac{1}{2}\Big) \kappa_H \pi_{2, H} +
\Big(p_{H^c}-\frac{1}{2}\Big) \kappa_{H^c} \pi_{2, H^c} & \text{if}~ j=2
\end{cases} \nonumber\\
&=\begin{cases}
\Big(\frac{1}{2}-p_H \Big) \kappa_H \pi_{1, H} +
\Big(\frac{1}{2}-p_{H^c} \Big) \kappa_{H^c} \pi_{2, H}  & \text{if}~ j=1  \\
\Big(p_H-\frac{1}{2}\Big) \kappa_H \pi_{2, H} +
\Big(p_{H^c}-\frac{1}{2}\Big) \kappa_{H^c} \pi_{1, H}  & \text{if}~ j=2
\end{cases} \label {eq:nu_term}
\end{align}
Now, recall that the strategies $p$ under consideration here have $p_t= p_t^m$ for $t < \tau$, and that
$p_t^m$ is determined by the sign of $\xi^r_t$. If we treat $\xi^r_t=\xi^r (g_{I_{-T:t-1}})$
as a function of history, it is straightforward to see that when $H$ and $H^c$ are complementary,
\[
\xi^r (g_{I_{-T:t-1}}) = -\xi^r (g_{J_{-T:t-1}}).
\]
(It is important here that $p^m(\xi_t^r) = \left(\frac{1}{2},\frac{1}{2} \right)$ if $\xi_t^r=0$.)
Thus, $p^m (\xi^r(g_{I_{-T:t-1}}))$ chooses arm $1$ whenever $p^m (\xi^r(g_{J_{-T:t-1}}))$
chooses arm $2$, and vice versa. It follows that for the strategies under consideration,
\[
\text{Prob}_{p_{t-1}} (I_{t-1}|H_{t-2}) =
\text{Prob}_{p_{t-1}} (J_{t-1}|H^c_{t-2})
\]
for $t \leq \tau$, and therefore
\begin{align}
\kappa_H = \kappa_{H^c} \label{eq:k_I_k_J}.
\end{align}

We now apply these observations to identification of the optimal $p$ for \eqref{eq:max-min-of-nu-tau}, which
by \eqref{eq:expected-gain-explicit} amounts to
$$
\max_p \min_j \sum_{H_{\tau-1}}
\Big(\frac{1}{2} -\eps (-1)^{j}  \Big(p_{\tau,1}-\frac{1}{2}\Big) \Big )
\text{Prob}_{a(j), p} (H_{\tau-1} ) .
$$
Only the term with a factor of $(-1)^j$ depends on $p$, so it suffices to consider
$$
\min_p \max_j \sum_{H_{\tau-1}}
\eps (-1)^{j}  \Big(p_{\tau,1}-\frac{1}{2}\Big)
\text{Prob}_{a(j), p} (H_{\tau-1} ) .
$$
Grouping the histories into complementary pairs and using \eqref{eq:nu_term} combined with \eqref{eq:k_I_k_J},
we see that this problem can be written in the form
\begin{align*}
\min_{0\leq p_{H_i}, p_{H_i^c} \leq 1}  \max \left ( \substack{
\sum_{i}
\kappa_{H_i} \Big(\frac{1}{2}-p_{H_i}\Big)\pi_{1, H_i} + \kappa_{H_i}  \Big(\frac{1}{2}-p_{H_i^c}\Big)  \pi_{2, H_i} \\
\sum_{_i}
\kappa_{H_i}  \Big(p_{H_i}-\frac{1}{2}\Big) \pi_{2, H_i} +\kappa_{H_i}   \Big(p_{H_i^c}-\frac{1}{2}\Big)  \pi_{1, H_i}} \right),
\end{align*}
where the summation is over all pairs of complementary strategies (chosen so that each strategy appears
just once).  Here the subscript $i$ indexes all possible histories through time $\tau-1$ but we omit the dependence of $H_i$ and $H^c_i$ on $\tau-1$ for simplicity. One easily sees that this optimization fits the conditions of \cref{lemma:minimax} below,
if for a given pair of complementary histories $H_i, H^c_i$ through time $\tau-1$  we take
$x_i =\frac{1}{2}- p_{H_i}$, $y_i = \frac{1}{2}-p_{H^c_i}$, $a_i =\kappa_{H_i} \pi_{1, H_i}$, and
$b_i =\kappa_{H_i} \pi_{2, H_i}$. 

\begin{lemma}\label{lemma:minimax} Let $a$ and $b$ be arbitrary vectors in $\R^d$. Then
\begin{displaymath}
\min_{-1/2 \leq x_i, y_i \leq 1/2} \max \Big (
 \substack{\langle x, a \rangle+ \langle y, b  \rangle, \\
 -\langle x, b \rangle- \langle y, a \rangle}  \Big)
 \end{displaymath}
is achieved when
\[
\begin{cases}
x_i^* = -1/2, y_i = 1/2 & \text{if} ~a_i >b_i\\
x_i^* +y_i^* = 0 & \text{if}~ a_i =b_i\\
x_i^* = 1/2, y_i^* = -1/2 & \text{if} ~a_i <b_i.
\end{cases}
\]
Moreover, at any optimal $(x,y)$ the values of $\langle x, a \rangle+ \langle y, b  \rangle$ and
$ -\langle x, b \rangle- \langle y, a \rangle$ are equal.
\end{lemma}
\begin{proof}
Since for any real valued $f$ and $g $, $\max (f, g) = \frac{1}{2} (f+g) + \frac{1}{2} |f-g|$, we have
\begin{align}
& \max \Big (
 \substack{\langle x, a  \rangle+ \langle y, b  \rangle, \\
 -\langle x, b   \rangle- \langle y, a \rangle}  \Big)=  \frac{1}{2}\Big  ( \langle x- y, a -b \rangle +  |  \langle x+y , a +b \rangle   | \Big) . \label{eq:minmax_objective}
\end{align}
It suffices to consider $(x,y)$ such that $x_i + y_i = 0$ for each $i$. Indeed, for any admissible
$x$ and $y$, the vectors $x' = (x - y)/2$ and $y' = (y-x)/2$ are also admissible,
and $x'-y' = x-y$ while $x'+y' = 0$, so the value of our objective at $(x',y')$ is at least as good as the value
at $(x,y)$. The assertion of the lemma is now clear, by optimizing the linear function $\langle x-y, a-b \rangle$.
(We remark -- though this will not be used -- that the $x_i^*$ and $y_i^*$ identified above are in fact the
only optimal choices, except that when $a_i = b_i = 0$ then $x_i$ and $y_i$ can
take any admissible value.)
\end{proof}

The lemma shows that an optimal strategy is obtained by taking $p_H=1$ and $p_{H^c}=0$ if $\pi_{1, H} > \pi_{2, H}$,
$p_H=1/2$ and $p_{H^c}=1/2$ if $\pi_{1, H} = \pi_{2, H}$, and $p_H=0$ and $p_{H^c}=1$ if $\pi_{1, H} <\pi_{2, H}$.\footnote{Since $\kappa_H\geq 0$, the ordering of $a_i =\kappa_H \pi_{1, H}$, and $b_i =\kappa_H \pi_{2, H}$
is the same as ordering of $\pi_{1, H}$, and $ \pi_{2, H}$ when $\kappa_H> 0$. When $\kappa_H= 0$, the
ordering of $a_i$ and $b_i$ does not matter.} Essentially, this strategy chooses the 
arm $i$ for which $\pi_{i, H}$ is larger. Since
\[
\pi_{1, H}/{\pi_{2, H}} =  \left(\frac{1+ \epsilon}{1 - \epsilon} \right)^{\xi^r}
\]
the optimal strategy just identified is in fact $p_\tau^m$. The lemma also assures us that this strategy makes the player
indifferent (through time $\tau$) to the choice of the safe arm $j$.

As noted earlier, after repeating this argument finitely many times, we conclude that it is optimal to use
the strategy $p^m$ at every time (through $t=-1$), and that the final-time
regret does not depend upon which arm is safe (in other words, $R_T(p^m,a(1)) = R_T(p^m,a(2))$.

The proof that $p^m$ is also optimal the context of pseudoregret is essentially the same, so we omit it.
 
\section{Proof of \cref{lemma:pde_soln}}
\label{app:pde_soln}   
After a change of coordinates $z=\xi^h +  \xi^r$, \cref{eq:homogen_pde} becomes a 1D heat equation
\begin{subequations}
\begin{align*}
&u_t + \kappa u_{zz} = 0 \\
&u(\eta, z,0) = \frac{1}{2}(\eta +|z|)  
\end{align*}
\end{subequations}
and its unique smooth solution $u^h$ is therefore \cref{eq:phi}. Since $\varphi$ does not depend on $\eta$ or $\xi^h$,  \cref{eq:homogen_pde1}  with the final value $\varphi(\xi^r, 0)$ is also a 1D heat equation
 \begin{subequations}
\begin{align*}
&u_t + \frac{\kappa}{2}u_{\xi^r \xi^r} = 0 \\
&u(\xi^r,0) =\varphi(\xi^r, 0)  
\end{align*}
\end{subequations}
and its unique smooth solution $ \hat \varphi$ is therefore \cref{eq:hatphi}.

\section{Proof of \cref {lemma:error}}
 \label {app:error}

\subsection{Derivatives of $u^h$}

Since we can put one derivatives under the integral on the absolute value function and the remaining derivatives on the fundamental solution $\Phi$, for $d \geq 1$,
\begin{align}
|\partial^{d}_z u^h| &= \frac{1}{2}\left | \int_{\R} \partial^{d-1}_z\Phi(z  - s,2t)  \partial_s  |s |ds\right | \leq   \frac{1}{2} \int_{\mathbbm R} |\partial^{d-1}_s \Phi (s, 2t)| ds = O \Big( |\kappa t  |^{\frac{1-d}{2}} \Big). \label{eq:u_h_bound}
\end{align}
Since $u^h_t = -{\kappa} u^h_{zz}$,  we have $u^h_{tt} = \kappa^2\partial^{4}_z u^h$. Therefore, $u^h_{tt}  =O\Big (\sqrt {\kappa} |t|^{-\frac{3}{2}}  \Big)$. 

\subsection{Derivatives of $\hat \varphi$}
It is elementary that  $\varphi'     = O(1 + b  \epsilon/ \sigma)$,  and for $d\geq 2$, $\varphi^{(d)}  = O (b (\epsilon/\sigma)^d)$.  Similarly to \eqref{eq:u_h_bound}, for $d \geq 1$, we can put one derivative on $\varphi$
\begin{align}
\partial^d_{\xi^r} \hat \varphi     =O \left( (1 + b \epsilon /\sigma)  |\kappa t |^{\frac{1-d}{2}}\right)~ \text{and}~ \hat \varphi_{tt}     =O \left( (1 + b  \epsilon/\sigma) \sqrt {\kappa} | t |^{-\frac{3}{2}}\right). \label {dhatvarphi} 
\end{align}

However if $b =\kappa/ \sigma$, i.e. $\varphi$ is $C^1$, and we can put two derivatives on $\varphi$: for $d \geq 2$, 
\begin{align*}
\left |\partial^d_{\xi^r} \hat \varphi (\xi^r, t)\right |&=  \left |\int_{\R} \partial^{d-2}_{\xi^r} \Phi(\xi^r-s, t) \varphi'' (s) ds \right| \leq  \max_{s \in \R } |\varphi''(s)| \int_{\mathbbm R} |\partial^{d-2}_{s}\Phi (s,t)| ds =O \left( \frac{\epsilon}{\sigma}  | \kappa t |^{1-\frac{d}{2}}\right)
\end{align*}
and since $\int_{\mathbbm R} |\varphi''(s)| ds = O\Big(b \frac{\eps}{\sigma}\Big)$, and $\partial^{d}_{s}\Phi (s, t) =O( |\kappa t|^{-\frac{d+1}{2}})$,
\begin{align*}
&|\partial^d_{\xi^r} \hat \varphi (\xi^r, t) |  \leq  \max_{s \in \R} |\partial^{d-2}_{s}\Phi (s,t) | \int_{\mathbbm R} |\varphi''(s|) ds =O \left( b\frac{\eps}{\sigma}  |\kappa t |^{\frac{1-d}{2}}\right)=O \left( \kappa^{\frac{1-d}{2}}  | t |^{\frac{1-d}{2}}\right).
\end{align*}
Therefore, for $d \geq 2$,
\begin{align}
\partial^d_{\xi^r} \hat \varphi  = O \left( \min\Big( \frac{ \epsilon}{\sigma},  | t |^{-\frac{1}{2}} \Big)|\kappa t |^{1-\frac{d}{2}} \right). \label {dxirhatvarphi} 
\end{align}
Also since  $\hat \varphi_{tt} =  \frac{\kappa^2}{4} \partial^4_{\xi^r} \hat \varphi$, we have
\begin{align*}
\hat \varphi_{tt} = O \left( \min\Big( \frac{\epsilon}{\sigma}  , | t |^{-\frac{1}{2}} \Big)| \kappa t |^{-1} \right).  
\end{align*}

\section{Proof of \cref{thm:lb}} 
\label{app:lb}
We will show that 
\[
|u(\eta,\xi, t)- v(\eta,\xi t)| \leq E_1(t)
\] where $E_1$ is given by \cref{eq:et} in two steps.  In the first step, we  establish the upper bound: 
\begin{align} \label{eq:u_lb}
\mathbb E_{a,p^m} ~u(\eta+d\eta,\xi+d\xi , t+1)- u (\eta,\xi, t) \leq K(t)
\end{align}
uniformly in $\eta$ and $\xi$ where $K(t)$ is given by \cref{eq:K} and \cref{eq:Kasy}. Since $u^h$ and therefore $u$ is not differentiable at $t=0$ and $\xi^r + \xi^h =0$, in \cref{app:final_period} we consider the final prediction period separately from the earlier periods.  Also \cref{app:periods_before_final},  we will treat separately the region where $\xi^r +\eps t > 0$ or $\xi^r +\eps t < 0$  where $u$ is smooth (\ref{app:positive_xir} and \ref{app:negative_xir}) from the region where $\xi^r+\eps t =0$   where $\varphi''$ and therefore $u_{\xi^r \xi^r}$ are discontinuous (\ref{app:zero_xir}).  Since the lower bound 
\[
- K(t) \leq \mathbb E_{a,p^m} u (\eta+d\eta,\xi+d\xi , t+1)- u  (\eta,\xi, t) 
\]
can be proved similarly to the upper bound, we omit the proof of lower bound to avoid repetition. 

 The second step connects $u$ and $v$. Since $v$ is defined by the iterative scheme \eqref{eq:w_dp},  in \cref{app:step2}, we show that $u(\eta,\xi, t)- v(\eta,\xi, t) \leq E_1(t)$ by induction starting from the final time.  The proof that $- E_1(t)\leq u(\eta,\xi, t)-v(\eta,\xi, t) $ is similar and therefore is omitted.

\subsection{Final period} 
\label{app:final_period} We consider the evolution of $u$ during the final prediction period  ($t$ changes from $-1$ to $0$). Since $u^h(\eta, z,0) =  \mu (\eta, z)$, 
\begin{align}
&\left |\mu(\eta+ d\eta, z+ dz) - u^h(\eta, z, -1)\right|  \leq \Big |\mu(\eta+ d\eta, z+ dz) - \mu(\eta, z)\Big | + \left| \mu(\eta, z) - u^h(\eta, z, -1)\right| \label{eq:u_t_eq_0}
\end{align}
is bounded  above uniformly in $\eta$, $z$ and $\epsilon$. Since the absolute values of $d\eta$ and $dz$ are uniformly bounded, then so is $|\mu(\eta+ d\eta, z+ dz) - \mu(\eta, z)|$. Also since $ -| z-s|  \geq -|z| -|s|$, we obtain
\begin{align*}
& \mu(\eta, z) - u^h(\eta,z, -1)   =    \mu (\eta, z) - \int_{\mathbbm R} \Phi(s, -1) \mu (\eta, z-s) ds\\
 &=  \int_{\mathbbm R} \Phi(s, -1) (|z| - |z-s|) ds \geq  - \int_{\R} \Phi(s, -1) |s| ds
\end{align*}
which is uniformly bounded from below. It is also bounded uniformly from above since $ - |z-s| \leq -| z| +  | s|$.  Therefore,  \cref{eq:u_t_eq_0}
is bounded above by a constant uniformly in $\eta$, $\xi$ and $\epsilon$.  

Arguing as in the previous paragraph, we have
\begin{align*}
&\left |u^n(\xi^r+ d\xi^r,  0) - u^n(\xi^r, -1)\right|  \leq \Big |u^n(\xi^r+ d\xi^r, 0) - u^n(\xi^r, 0)\Big | + \left| u^n(\xi^r, 0) - u^n(\xi^r, -1)\right| 
\end{align*}
The first term vanishes since $u^n(\xi^r, 0)  = 0$, while the second term is bounded uniformly in $\xi^r$ since
\begin{align*}
&\left | u^n( \xi^r, -1) \right |= \left |\int_{\R} \Phi( \xi^r-s,-1) (\varphi(\xi^r)-\varphi(s)) ds \right |\\
& \leq \max_{s\in \R} \left |\varphi'(s) \right | \int_{\R} \Phi( \xi^r-s,-1) \left |s-\xi^r \right| ds =  \max_{s\in \R} \left |\varphi'(s) \right | \int_{\R} \Phi( s,-1) \left |s \right| ds.
\end{align*} 

\subsection{Periods before the final one}
 \label{app:periods_before_final} Now we consider the evolution of $u$ before the final prediction period, i.e., at $t \leq -2$.  By the rules of the game $\xi^r$ only takes integer values.  Since $u(\eta+c, \xi, t) =u(\eta, \xi, t) +c/2$ for any $c \in \R$,
\begin{align*}
&\mathbb E_{a, p^m} ~u  (\eta+d\eta,\xi+d\xi , t+1)\\
&=\begin{cases}
- \epsilon +\mathbb E_{a} ~u(\eta,  \xi^h-g_2-\eps, \xi^r+ g_1+\eps,t+1)  &\text {if}~\xi^r +\eps t \geq 1  \\
  \frac{1}{2}  \mathbb E_{a} [ u(\eta, \xi^h-g_2-\eps, \xi^r+ g_1+\eps,t+1)+ u(\eta, \xi^h+g_1-\eps,\xi^r-g_2-\eps, t+1) ]&\text {if} ~\xi^r+\eps t = 0\\
\epsilon+ \mathbb E_{a} ~u(\eta, \xi^h+g_1-\eps,\xi^r- g_2-\eps, t+1) &\text {if}~\xi^r +\eps t\leq -1
 \end {cases} \\
 &= \mathbb E_{a} ~u(\eta,  \xi^h-g_2 -\eps, \xi^r+ g_1-\eps,t+1)  +\begin{cases}
- \epsilon &\text {if}~\xi^r +\eps t\geq 1  \\
0&\text {if} ~\xi^r+\eps t = 0\\
\epsilon &\text {if}~\xi^r+\eps t \leq -1
 \end {cases} 
\end{align*} 
where the last equality holds because the laws of $-g_2$ and $g_1$ are the same.  We consider
\begin{subequations}{\label{eq:u_evolution}}
 \begin{align}
&u(\eta,  \xi^h-g_2-\eps, \xi^r+ g_1-\eps,t+1)- u(\eta,  \xi, t) \nonumber\\
&= u^h (\eta,\xi^h + \xi^r + g_1-g_2-2 \eps,t+1) - u^h (\eta,\xi^h + \xi^r ,t) \label{eq:uh_evolution}\\
&-\hat \varphi ( \xi^r + g_1- \eps, t+1) +\hat \varphi   ( \xi^r ,t) \label{eq:hat_phi_evolution} \\
&+ \varphi ( \xi^r + g_1+  \eps t) - \varphi   ( \xi^r + \eps t) \label{eq:phi_evolution}
\end{align} 
\end{subequations}
By Taylor expansion \cref{eq:uh_evolution} is given by
 \begin{align*}
&  A:= u^h_t + (g_1-g_2-2 \eps) u_z^h + \frac{1}{2} (g_1-g_2-2 \eps)^2 u_{zz}^h + K^h
\end{align*} 
where all the derivatives are evaluated at $(\eta,\xi^h + \xi^r ,t+1)$  and 
\begin{align*}
&K^h  = \frac{1}{6}(g_1-g_2-2 \eps)^3 \partial^3_z u^h(\xi^h+\xi^r,t+1)\\
&+ (g_1-g_2-2 \eps)^4\int_0^1 \partial^4_z u^h (\xi^h+\xi^r+\mu (g_1-g_2-2 \eps) ,t+1)\frac{(1-\mu)^3}{6}d\mu  -\int_0^1 u^h_{tt} (\eta, \xi, t+ \mu)(1-\mu)d\mu 
\end{align*}
Since  the expectation of the terms involving $u_z$ is zero and  $\E_a [ (g_1-g_2-2 \eps)^2]= 2\kappa $,
 \begin{align*}
& \E_a [A]:= u^h_t  + \kappa u_{zz}^h + \E_a [ K^h].
\end{align*} 
Since the expectation of the third order term  is also zero  and $\E_a [ (g_1-g_2-2 \eps)^4]= 8\kappa $, by \cref {lemma:error}
\begin{align*}
\E_a [K^h]  =& O(\kappa \partial^4_z u^h+ u^h_{tt}) = O(\sqrt {\kappa} |t|^{-\frac{3}{2}} )
\end{align*}  Finally,  since $u^h_t  + \kappa u_{zz}^h =0$,  we have
 \begin{align}
& \E_a [A]=  O(\sqrt {\kappa} |t|^{-\frac{3}{2}} ). \label{eq:EA}
\end{align}
Similarly \cref{eq:hat_phi_evolution} is given by
 \begin{align*}
&  B:= \hat \varphi_t+  (  g_1- \eps)\hat \varphi_{\xi^r} + \frac{1}{2} (g_1- \eps)^2 \hat \varphi_{\xi^r\xi^r} + \hat K
\end{align*} 
where all the derivatives are evaluated at $(\xi^r ,t+1)$  and 
\begin{align*}
\hat K  =& \frac{1}{6}(g_1- \eps)^3 \partial^3_{\xi^r}  \hat \varphi(\xi^r,t+1)+ (g_1- \eps)^4\int_0^1 \partial^4_{\xi^r} \hat \varphi (\xi^r+\mu (g_1- \eps) ,t+1)\frac{(1-\mu)^3}{6}d\mu \\
& -\int_0^1 \hat \varphi_{tt} ( \xi^r, t+ \mu)(1-\mu)d\mu 
\end{align*}
Since  the expectation of the first order terms is again zero, and  $\E_a [ (g_1- \eps)^2]= \kappa $
 \begin{align*}
& \E_a [B]:=\hat \varphi_t  + \frac{\kappa}{2} \hat \varphi_{\xi^r\xi^r} +  \E_a [ \hat K].
\end{align*} 
Since  the expectation of the third order term in $\hat K$ is zero  and $\E_a [ (g_1- \eps)^4]= \kappa (3\eps^2+1)$, for $b = \frac{ \kappa}{\eps}$, by \cref {lemma:error},
\begin{align}
\E_a [\hat K]  =& O(\kappa \partial^4_{\xi^r} \hat \varphi+ \hat \varphi_{tt}) = O( \min\Big( \epsilon , \kappa | t |^{-\frac{1}{2}}\Big ) | t |^{-1}). \label{eq:EhatK}
\end{align} Thus, using $\hat \varphi_t  +  \frac{\kappa}{2}  \hat \varphi_{\xi^r\xi^r} =0$, we have
 \begin{align*}
& \E_a [B]=  O( \min\Big( \epsilon , \kappa | t |^{-\frac{1}{2}}\Big ) | t |^{-1})
\end{align*} 
as well.

It remains to consider the evolution of $\varphi$. Since  $\varphi$ is at most $C^1$ at $\xi^r + \eps t=0$, we consider the following 3 cases. 

\subsubsection{ $\xi^r + \eps t \geq 1$}  \label{app:positive_xir} The function $\varphi$ is $C^\infty$ for $|\xi^r + \eps t| >0$ and $t<0$. Therefore, we can use its Taylor's expansion of \cref{eq:phi_evolution}. Since  $\varphi ''(y) = 4b (\frac{\eps}{\sigma})^2 \exp (-\frac{2\eps}{\sigma} y)$, 
\begin{align*}
&C:= \varphi ( y+ g_1) - \varphi   (y)  = g_1\varphi' (y) + g_1^2 \int_0^1  \varphi'' (y+ \mu g_1)(1-\mu)d\mu\\
 & = g_1\varphi' (y) +\zeta \varphi ''(y) 
\end{align*} 
where
\[
\zeta =  \int_0^1  \exp \Big(-\frac{2\eps}{\sigma} \mu g_1\Big)(1-\mu)d\mu = \frac{1}{4(\frac{\eps}{\sigma}g_1)^2} \Big(\frac{2\eps}{\sigma}g_1 +\exp\Big (-\frac{2\eps}{\sigma}  g_1\Big)-1\Big)
 \]

When $\mu =1$,  the above integrand is not defined at $ \xi^r+ \mu g_1 = 0$, i.e., when $\xi^r = 1$  and $g_1=-1$. However, since this occurs at an endpoint of the integration interval, we can simply ignore it for the purpose of evaluating the integral.  Therefore, 
 \begin{align}
&E_a [\zeta] =  \frac{1}{4(\frac{\eps}{\sigma})^2} \Big(\frac{2\eps}{\sigma}\eps +\exp\Big (-\frac{2\eps}{\sigma} \Big)(1+\eps)/2+\exp\Big (\frac{2\eps}{\sigma} \Big)(1-\eps)/2\Big) \nonumber\\
&=  \frac{1}{4(\frac{\eps}{\sigma})^2} \Big(\frac{2\eps}{\sigma}\eps +\exp\Big (-\frac{2\eps}{\sigma} \Big)(1+\eps)/2+\exp\Big (\frac{2\eps}{\sigma} \Big)(1-\eps)/2\Big) \nonumber \\
&=   \frac{\sigma}{2}  +\frac{1}{4(\frac{\eps}{\sigma})^2} \Big(\exp\Big (-\frac{2\eps}{\sigma} \Big)(1+\eps)/2+\exp\Big (\frac{2\eps}{\sigma} \Big)(1-\eps)/2 -1\Big) \label{eq:Ea_zeta}
 \end{align}
For $\sigma = \kappa$ we have  $\eps \varphi'  + \frac{\kappa}{2} \varphi '' =\eps$, and therefore
 \begin{align*}
  \E_a [C]&=  \eps \varphi'  + \frac{\kappa}{2} \varphi '' + b \nu (\eps) \exp (-\frac{2\eps}{\kappa} y)=\eps  +   K^+(t).
\end{align*}
where 
\begin{align}
K^+(t) =O \left(  b \nu (\eps) \right) \label{eq:K+}
\end{align} and 
\[
\nu(\eps) =  \Big(\exp\Big (-\frac{2\eps}{\kappa} \Big)(1+\eps)/2+\exp\Big (\frac{2\eps}{\kappa} \Big)(1-\eps)/2 -1\Big)
\]
For $b = \frac{\kappa}{\eps}$, 
\begin{align*}
K^+(t) =O \left(   \frac{\kappa}{\eps} \nu (\eps) \right). 
\end{align*}
As $\eps \rightarrow 0$, the leading order term  of $\nu$ is given by
 \begin{align*}
&\nu(\eps)  \approx    \frac{\eps^2}{\kappa^2} (2 - 2 \kappa)  =   2\frac{\eps^4}{\kappa^2}
 \end{align*}
 and therefore, 
\begin{align}
K^+(t) =O \left(   \frac{\eps^3}{\kappa} \right) \approx O \left(   \eps^3 \right). \label{eq:K+_C1}
\end{align}

\subsubsection{$\xi^r + \eps t \leq -1 $}     \label{app:negative_xir}  The function $\varphi(y) = -y$ is linear for $y<0$ .  Therefore, 
\begin{align*}
&C:= \varphi ( y+ g_1) - \varphi   (y)  = g_1\varphi' (y) =-g_1~\text{and}~ \E_a[C] = -\eps
\end{align*} 
where there is no error term ($K^-(t) =0$) as a result of the linearity of $\varphi$. 

\subsubsection{$\xi^r +\eps t =0$}  \label{app:zero_xir} 
When $\xi^r=0$, we must argue a little differently because $\varphi$ is only piecewise smooth in $\xi^r$. (Indeed $ \varphi $ is only $C^1$ at $\xi^r+\eps t=0$.) But our method still works using the explicit values of $\varphi (g_1)$. Since  $\varphi (0) =0 $, 
\begin{align*}
&C:= \varphi (  g_1) - \varphi   (0)  = \varphi (  g_1)
\end{align*} 
and we have
 \begin{align}
& \E_a [C]= \Big(1 + b  \exp \Big(-\frac{2\eps}{\sigma} \Big)-b\Big) (1+\eps)/2 + (1-\eps)/2 = K^0(t) \label{eq:E_C}
\end{align}
For $\sigma = \kappa$, 
\begin{align}
  K^0(t)  = b \rho(\eps)+1. \label{eq:K0}
\end{align}
where 
\[
\rho (\eps) = \Big(   \exp \Big(-\frac{2\eps}{\kappa} \Big)-1\Big) (1+\eps)/2 
\]
For $b = \frac{\kappa}{\eps}$, 
\begin{align}
  K^0(t)  =  \frac{\kappa}{\eps}\rho (\eps)+1. 
\end{align}

When $\eps \rightarrow 0$, the leading order term of  $ \frac{\kappa}{\eps} \rho (\eps)$ is $\Big(\frac{\eps}{\kappa}-1\Big) (1+\eps) $. Using the definition of $\kappa=1-\eps^2$
  \begin{align*}
& \E_a [C] \approx \Big(\frac{\eps}{\kappa}-1\Big) (1+\eps) + 1 =  K^0(t).
\end{align*}
where 
\begin{align}
K^0(t) =O \left(  \frac{\eps^2}{1-\eps} \right)=O \left( \eps^2 \right). \label{eq:K0_C1} 
\end{align}

Combining the foregoing,  for all  $\xi^r+\eps t$
\[
\mathbb E_{a,p} [u(\eta+d\eta,  \xi +d \xi^r,t+1)]- u(\eta,  \xi, t) \leq K (t) )
 \]
where 
\begin{align} 
K(t)& = \E_a[A]+ \E_a[B] +\max (K^+, K^0, K^-) \nonumber \\
&= O \left( \sqrt {\kappa} |t|^{-\frac{3}{2}}+ \min \Big( \epsilon , \kappa | t |^{-\frac{1}{2}}\Big ) | t |^{-1}+  \frac{\kappa}{\eps} \nu (\eps)+   \frac{\kappa}{\eps} \rho(\eps) + 1  \right). \label{eq:K}
\end{align}
When $\eps \rightarrow 0$, the leading order term of $K(t)$ is 
\begin{align} 
K(t)&= O \left(  |t|^{-\frac{3}{2}}+ \eps^2 \right).\label{eq:Kasy}
\end{align}

\subsection{Approximation of $v$ by $u$ by induction}
\label{app:step2}
Lastly, we show that $v \leq u + E_1(t) $ (where the discretization error $E_1$ is defined below)  by induction backwards from the final time. In doing so, we are proving the associated regret is approximately $u$. If one accepts the use of our myopic player, then the bandit problem can be viewed as a Markov chain with $(\eta, \xi^r, \xi^h)$ as its state space; in this setting the PDE for $u$ is the backwards Kolmogorov equation associated with the scaling limit of this Markov chain. 

This proof is similar in character to the proof of Theorem 3 in \cite{kobzar}.   Specifically,  initialization of the induction follows from the fact that $u(\eta, \xi,0) =v(\eta, \xi, 0) + E_1(0)$ where the function $E$ is given by
\begin{align}
E_1(t) = \begin{cases}
0 & t=0 \\
C & t=-1\\
C+ \sum _{\tau = t}^{-2}  K(\tau) & t \leq -2
\end{cases} \label{eq:et}
\end{align}
for a constant $C$.  The inductive hypothesis is that
\[
 v(\eta, \xi, t+1) \leq  u(\eta,\xi ,t+1) +E_1(t+1) 
\]
Since $K(t) =E_1(t) - E_1(t+1)$,
\begin{subequations}{\label{eq:induct}}
\begin{align*}
 u  ( \eta,\xi, t) +E_1(t) &\geq  \mathbb E_{p^m,a} ~u  (\eta+d\eta,\xi+d\xi ,t+1) + E_1(t+1)  ~~~~~ \textbf{ [by \cref{eq:u_lb} ]} \\
& \geq   ~ \mathbb E_{p^m,a} ~ v(\eta+d\eta,\xi+d\xi, t+1)~~~~~~~~\textbf {[by the hypothesis]}\\
&=v(\eta, \xi, t).~~~~~~~~~~~~~~~~~~~~~~~~~~~~~~~~~~~~~~~~~~~~~~ \textbf{ [by \cref{eq:w_dp_b}] }
\end{align*}
\end{subequations}
We estimate $\sum_{\tau =t}^{-1} K(\tau)$ by an integral. We first consider the term $\min \Big( \epsilon , \kappa | t |^{-\frac{1}{2}}\Big ) | t |^{-1}$. Integrating from $t = -1$ to $- \tau$ such that $\epsilon = \kappa \tau ^{-\frac{1}{2}}$, i.e., where  $\epsilon \leq \kappa |t| ^{-\frac{1}{2}}$, and also separately from $ \tau -1$ to $-T$ where $\epsilon > \kappa |t| ^{-\frac{1}{2}}$ (or from $t = -1$ to $-T$ if  $\epsilon < \kappa T ^{-\frac{1}{2}}$) leads to the following cumulative error estimate
\begin{align}
 O \Big(  \epsilon\min ( \log (\kappa^2/ \eps ^2) +   1, \log T)\Big) \label{eq:min_estimate}
\end{align}
attributable to that term. Together with the other terms, we obtain 
\[
E_1(t) =O \Big(1+\sqrt {\kappa} +   \epsilon \min ( \log (\kappa^2/ \eps ^2) +   1, \log T) +  \big(\frac{\kappa}{\eps} \nu (\eps)+     \frac{\kappa}{\eps}\rho (\eps) +1\big)|t| \Big).
\]
When $\eps \rightarrow 0$, the leading order term of $E_1(t)$ is 
\[
O \Big(1+ \sqrt {\kappa} + \big(\frac{\eps^3}{\kappa}+  \frac{\eps^2}{1-\eps}\big)|t| \Big)=O \Big(1+ \eps^2 |t| \Big).
\]
 
\section{Proof of \cref{lemma:c}}
\label{app:c}
We will use the following function $f$ to analyze $u$
\begin{align}
\label{eq:sturm_ode_soln}
f(z) = \sqrt{\frac {2}{ \pi}} e^{-\frac{z^2}{2}} +z \text{erf} \left ( \frac{z}{\sqrt{2}} \right) ~~&\text{and} ~~\text{erf}(y)  = \frac {2}{\sqrt \pi }  \int_0^{y } e^{-s^2} ds.
\end{align}
As shown in Appendix J of \cite{kobzar}, $f$ solves 
$f(z) =  f''(z) + z f'(z)$ with  $\lim_{|z| \rightarrow \infty} \frac {f(z)}{ |z|}=1$. Therefore, $g(x,t) = \sqrt {-2 \kappa t} f \big (\frac {x}{\sqrt{-2 \kappa t}} \big)$ where $g(x,t)$ solves the 1D linear heat equation on $\mathbb R \times \mathbb R_{<0}$:
$ g_t + \kappa g_{xx} =0$ with $g(x,0) =|x|$. Therefore, $u^h$ can be expressed as: 
\[
u^h(\eta, z,t)   =\frac{1}{2} \left(\eta + \sqrt {-  2\kappa t} f \left (\frac {z}{\sqrt{- 2 \kappa t}} \right) \right )
\]
and
\[
\frac {1}{\sqrt {T} } u^h(0, 2\eps T, -T)  = \sqrt{ \frac{\kappa} {\pi} } \exp \left( - \frac{\epsilon^2T}{\kappa} \right)   +\epsilon \sqrt {T}~ \erf \left ( \epsilon \sqrt {\frac{T  }{\kappa}} \right). 
\]
Also
\[
\hat \varphi(\xi^r,t) =\int_{\R}  \Phi(\xi^r-s, t) \varphi (s) ds =  \sqrt {-  \kappa t} f \left (\frac {\xi^r}{\sqrt{-  \kappa t}} \right)+ b\int_{0}^{\infty}  \Phi(\xi^r-s, t) e^{-2 \frac{\epsilon}{ \sigma} s} ds -b \int_{0}^{\infty}  \Phi(\xi^r-s, t)  ds 
\]
where
\begin{subequations}{\label{eq:int_phi_n}}
\begin{align}
&\int_0^\infty  \Phi(\xi^r- s ,t) ds   = -  \int_{\frac{\xi^r}{\sqrt {-2 \kappa t}}}^\infty  e^{-s^2} ds = \frac{1}{2}\Big(1+ \text{erf} \Big({\frac{\xi^r}{\sqrt {-2 \kappa t}}}\Big)\Big), ~\text{and} \\
&\int_0^\infty  \Phi(\xi^r- s ,t) e^{-2 \frac{\epsilon}{\sigma} s} ds  = \frac{1}{2}e^{-2 \frac{\epsilon}{\sigma^2} (\sigma \xi^r +\eps \kappa t)} \Big(1+\text{erf}  \left (\frac{\sigma \xi^r+2\epsilon \kappa t}{\sigma\sqrt {-2\kappa t}} \right)\Big).
\end{align}
\end{subequations}
Therefore, 
\begin{align*}  
\frac{1}{\sqrt{T}}\hat \varphi ( \eps T,-T) =&  \sqrt { \frac{2 \kappa}{ \pi}} e^ {-\frac{\epsilon^2 T}{2\kappa}} + \left( \epsilon \sqrt {T} -\frac{b}{  2 \sqrt {T}}\right)\text{erf}  \left (\epsilon \sqrt{\frac{T}{ 2\kappa}} \right)  +  \frac{b}{2}\Big(e^{-2 \frac{\epsilon^2}{\sigma^2} (\sigma -  \kappa) T}\Big(1+\text{erf}  \left (\frac{ \eps (\sigma -2  \kappa) \sqrt {T}  }{\sigma\sqrt {-2\kappa }} \right)\Big) -1\Big).
\end{align*} 
Combining the foregoing results we obtain
\begin{align}
 \frac {1}{\sqrt {T} } u (0, \eps T \mathbbm 1 ,-T) =&  \sqrt{ \frac{\kappa} {\pi} } \exp \left( - \frac{\epsilon^2T}{\kappa} \right)   +\epsilon \sqrt {T}~ \erf \left ( \epsilon \sqrt {\frac{T  }{\kappa}} \right) +\left(\frac{b}{ 2 \sqrt{T}} - \epsilon \sqrt{T} \right) ~\erf  \left (\epsilon  \sqrt{\frac{T}{2\kappa}} \right)    \nonumber \\
 &   -  \sqrt {\frac{2\kappa}{  \pi}} \exp \left(- \frac{\epsilon^2 T}{2\kappa} \right)+-  \frac{b}{2}\Big(e^{-2 \frac{\epsilon^2}{\sigma^2} (\sigma -  \kappa) T}\Big(1+\text{erf}  \left (\frac{ \eps (\sigma -2  \kappa) \sqrt {T}  }{\sigma\sqrt {-2\kappa }} \right)\Big) -1\Big). \label {eq:u0}
 \end{align}
When $\epsilon \rightarrow 0$  and 
$\gamma = \epsilon \sqrt {T}$  , since $\kappa \rightarrow 1$ and $\sigma \rightarrow 1$ , the  assertion of  \cref{lemma:c_sigma} follows if the leading order term of $b$ is $\frac{1}{\eps}$. 
 
When $\kappa = \sigma$,
\begin{align}
 \frac {1}{\sqrt {T} } u (0, \eps T \mathbbm 1 ,-T) =&  \sqrt{ \frac{\kappa} {\pi} } \exp \left( - \frac{\epsilon^2T}{\kappa} \right)   +\epsilon \sqrt {T}~ \erf \left ( \epsilon \sqrt {\frac{T  }{\kappa}} \right) \nonumber \\
 & +\left(\frac{b}{ \sqrt{T}} - \epsilon \sqrt{T} \right) ~\erf  \left (\epsilon  \sqrt{\frac{T}{2\kappa}} \right)    -  \sqrt {\frac{2\kappa}{  \pi}} \exp \left(- \frac{\epsilon^2 T}{2\kappa} \right). \label {eq:u0}
 \end{align}
When $\epsilon \rightarrow 0$  and 
$\gamma = \epsilon \sqrt {T}$, since $\kappa \rightarrow 1$, the assertion of  \cref{lemma:c} follows if the leading order term of $b$ is $\frac{1}{\eps}$

 \section{Proof of \cref{thm:lb2}} 
\label{app:lb2}
We shall show that by making a slightly different choice of the constant $b$ in the
definition of $\varphi$ (so that $\varphi$ and $u$ are no longer $C^1$ at $\xi^r = 0$) and setting $\sigma \neq \kappa$,
the arguments we used in \cref{app:lb} give the same leading-order estimate for 
the final-time regret, with a better error term. 

When $\varphi$ is not $C^1$ at $0$, the computation of  
 \begin{align*}
& \E_a [A]=  O(\sqrt {\kappa} |t|^{-\frac{3}{2}} ).
\end{align*} 
in \cref{eq:EA} is unchanged. However, instead of \cref{eq:EhatK}, we have
\begin{align*}
\E_a [\hat K]  =& O(\kappa \partial^4_{\xi^r} \hat \varphi+ \hat \varphi_{tt}) = O((1+b \frac{\eps}{\sigma}) \sqrt {\kappa} |t|^{-\frac{3}{2}} )
\end{align*} and therefore
 \begin{align*}
& \E_a [B]=  O((1+b\eps{\sigma}) \sqrt {\kappa} |t|^{-\frac{3}{2}} )
\end{align*} 

Next in  \cref{app:positive_xir} we modify the calculation of $K^+$ after \eqref{eq:Ea_zeta} as follows.   To determine the value of $\sigma$ that eliminates the discretization error, we set ${\sigma} = E_a [\zeta] $. This leads to
      \begin{align*}
\exp\Big (-\frac{2\eps} {\sigma} \Big)(1+\eps)/2+\exp\Big (\frac{2\eps}{\sigma} \Big)(1-\eps)/2= 1
 \end{align*}
which is solved by
\begin{align*}
\sigma = 2\eps \Big / \log  \Big(\frac{1+\eps}{1-\eps} \Big) 
 \end{align*}
With this choice of $\sigma$,  
 \[
 \eps \varphi'  + E_a [\zeta]  \varphi '' =\eps \varphi'  + \frac{\sigma}{2} \varphi ''=\eps
 \]
 Therefore, 
 \begin{align*}
  \E_a [C]&=  \eps \varphi'  + \frac{\sigma}{2} \varphi '' =\eps  +   K^+(t).
\end{align*}
where 
\begin{align}
K^+(t) =0  \label {eq:K+_C0}
\end{align}
Also the analysis in  \cref{app:negative_xir} is unchanged and we still have $K^-(t) =0$. 

Next in  \cref{app:zero_xir}, we modify the calculation after \eqref{eq:E_C} as follows. Since
\begin{align*}
\rho(\eps) =\Big(   \exp \Big(-\frac{2\eps}{\sigma } \Big)-1\Big) (1+\eps)/2 = \Big( \frac{1-\eps}{1+\eps} -1\Big) (1+\eps)/2 = -\eps
\end{align*}
\eqref{eq:K0} becomes
\begin{align}
  K^0(t)  = 1-\eps b. 
\end{align}
and for $b=1/\eps$
\begin{align}
  K^0(t)  = 0 \label {eq:K0_C0}. 
\end{align}

Combining the foregoing,  for all  $\xi^r+\eps t$, 
\[
\mathbb E_{a,p} [u(\eta+d\eta,  \xi +d \xi^r,t+1)]- u(\eta,  \xi, t) \leq K (t) 
 \]
where instead of \eqref{eq:K}, 
\begin{align*}
K(t) &= \E_a[A]+ \E_a[B] +\max (K^+, K^0, K^-)\\
 &= O \left(  (1+b\frac{ \eps}{\sigma}) \sqrt {\kappa} |t|^{-\frac{3}{2}}  \right)\\
  &= O \left( (1+\frac{1}{\sigma}) \sqrt {\kappa} |t|^{-\frac{3}{2}}   \right).
\end{align*}  
since $\max (K^+, K^0, K^-) =0$.
 Therefore, as $\eps \rightarrow 0$, we have
 \begin{align*}
K(t) &= O \left(  |t|^{-\frac{3}{2}}   \right).
\end{align*}  

Note that since the leading order behavior of $b$ as $\eps \rightarrow 0$ is unchanged  (it is still $1/\eps$), this choice of $b$ does not affect the leading-order behavior of $u(0 ,\eps T \mathbbm 1,-T)/\sqrt{T}$ as $T \rightarrow \infty$, i.e. the value of $c(\gamma)$ is unchanged.  But the errors  $ K^+$ for $\xi^r +\eps t \geq 1$ and  $K^0$ at $\xi^r +\eps t =0$  have been reduced as shown above from   $O(\eps^3)$ in \cref {eq:K+_C1} and $O(\eps^2)$ in \cref {eq:K0_C1} to zero in \cref{eq:K+_C0} and \cref{eq:K0_C0}. (The revised choice of $\varphi$ does not affect our arguments for $\xi^r +\eps t \leq -1$ where $K^-$ is still zero.) This improves the overall error since the only error term is now $O (|t|^{-\frac{3}{2}}) $. Bounding $\sum_{\tau =t}^{-1} K(\tau)$ by an integral, we obtain
\[
E_0(t) =O \left( 1+ (1+\frac{1}{\sigma}) \sqrt {\kappa}   \right).
\]
or, when $\eps \rightarrow 0$,
\[
E_0(t) =O \Big(1   \Big).
\]

\bibliographystyle{alpha}
\bibliography{expert_bounds_article}

\end{document}